\documentclass{article}

\usepackage{microtype}
\usepackage{graphicx}
\usepackage{subfigure}
\usepackage{booktabs} 

\usepackage{bbm}
\usepackage{xcolor}
\usepackage{amsmath}
\usepackage{amsthm}

\usepackage{thmtools}
\usepackage{thm-restate}

\usepackage{hyperref}

\usepackage{amssymb}
\usepackage{mathtools}
\usepackage{thmtools}
\usepackage{cleveref}
\usepackage{comment}
\usepackage[mathscr]{eucal}

\usepackage[accepted]{icml2021}

\usepackage{xcolor}

\newtheorem{theorem}{Theorem}
\newtheorem{lemma}{Lemma}
\newtheorem{claim}{Claim}
\newtheorem{definition}{Definition}

\newtheorem{condition}{Condition}

\DeclareMathOperator{\A}{{A}}

\DeclareMathOperator{\F}{{F}}
\DeclareMathOperator{\B}{{B}}
\DeclareMathOperator{\Sc}{\textit{{\text{F}}}}
\DeclareMathOperator{\I}{\mathcal{I}}

\DeclareMathOperator{\pit}{\tilde{\pi}}
\DeclareMathOperator{\Reg}{\text{Reg}}
\DeclareMathOperator{\Rew}{\text{Rew}}
\DeclareMathOperator{\mub}{\bar{\mu}}
\DeclareMathOperator{\muh}{\hat{\mu}}
\DeclareMathOperator{\OPT}{\text{OPT}}
\newcommand\CBBSD[1][]{CBBSD}

\newcommand\event[1]{\mathop{\mathbb{I}\left(#1\right)}}

\newcommand\Ex[2]{\mathop{\underset{#1}{\mathbb{E}}\left[#2\right]}}
\newcommand\Pro[1]{\mathop{\mathbb{P}\left(#1\right)}}

\begin{document}

\twocolumn[
\icmltitle{Combinatorial Blocking Bandits with Stochastic Delays}



\icmlsetsymbol{equal}{*}

\begin{icmlauthorlist}
\icmlauthor{Alexia Atsidakou}{equal,ece}
\icmlauthor{Orestis Papadigenopoulos}{equal,cs}
\icmlauthor{Soumya Basu}{goo}
\icmlauthor{Constantine Caramanis}{ece}
\icmlauthor{Sanjay Shakkottai}{ece}
\end{icmlauthorlist}

\icmlaffiliation{cs}{Department of Computer Science, The University of Texas at Austin, USA}
\icmlaffiliation{ece}{Department of Electrical and Computer Engineering, The University of Texas at Austin, USA}
\icmlaffiliation{goo}{Google, Mountain View, USA}

\icmlcorrespondingauthor{Alexia Atsidakou}{atsidakou@utexas.edu}
\icmlcorrespondingauthor{Orestis Papadigenopoulos}{papadig@cs.utexas.edu}

\icmlkeywords{blocking, bandits, combinatorial, upper, confidence, bound, stochastic, delays}

\vskip 0.3in
]



\printAffiliationsAndNotice{\icmlEqualContribution} 

\begin{abstract}
Recent work has considered natural variations of the {\em multi-armed bandit} problem, where the reward distribution of each arm is a special function of the time passed since its last pulling. In this direction, a simple (yet widely applicable) model is that of {\em blocking bandits}, where an arm becomes unavailable for a deterministic number of rounds after each play. In this work, we extend the above model in two directions: (i) We consider the general combinatorial setting where more than one arms can be played at each round, subject to feasibility constraints. (ii) We allow the blocking time of each arm to be stochastic. We first study the computational/unconditional hardness of the above setting and identify the necessary conditions for the problem to become tractable (even in an approximate sense). Based on these conditions, we provide a tight analysis of the approximation guarantee of a natural greedy heuristic that always plays the maximum expected reward feasible subset among the available (non-blocked) arms. When the arms' expected rewards are unknown, we adapt the above heuristic into a bandit algorithm, based on UCB, for which we provide sublinear (approximate) regret guarantees, matching the theoretical lower bounds in the limiting case of absence of delays.
\end{abstract}

\section{Introduction}
It is only recently that researchers have focused their attention on variants of the stochastic {\em multi-armed bandit} (MAB) problem where the mean reward of each arm is a specific function of the time passed since its last pulling~\cite{KI18, PBG19, CCB19}.
These variants capture applications where the mean outcome of an action temporarily decreases or fluctuates after each use. A simple yet very expressive model in this direction is that of {\em blocking bandits}~\cite{BSSS19}, where each arm becomes blocked (i.e., it cannot be played again) for a deterministic number of subsequent rounds after each play, known as the {\em delay}.

In this paper, we generalize blocking bandits to model two important and well-motivated  properties: (i) {\em Stochastic} blocking, where the delay of each arm is 
randomly sampled from a distribution after each pull, and (ii) {\em Combinatorial} actions, where more than one arm can be pulled at each round, subject to general feasibility constraints.

Our model can accommodate a variety of settings where arm pulling is subject to combinatorial constraints, while the repeated selection of a certain arm is undesirable or even infeasible. For example, in ride sharing platforms, where the users are matched with rides, the allocation of resources imposes {\em matching} constraints, while the transit times of rides give rise to {\em stochastic blocking} of resources \cite{DSSX18}. Matching constraints are also natural in task assignment in the cloud, and combinatorial network optimization where the processing/serving times of tasks can be modeled as stochastic blocking. Further, {\em knapsack} constraints emerge inherently in ad placement (multiple ad slots in a webpage) \cite{CWYW16}, and movie/song recommendation (multiple recommendation slots) \cite{KWAS15b}; while repetitive ads and recommendations should be avoided for user satisfaction by introducing (possibly) random delays.  Whereas these varied applications naturally are within the scope of our results, to the best of our knowledge, they are beyond the reach of prior works, and in particular~\cite{BSSS19,BPCS20,PC21,BCMT20}.

\subsection{Central Challenges and Main Contributions}
\noindent{\bf Model.} A {\em player} is given a set of {\em arms}, each corresponding to an unknown {\em reward-delay} distribution. At each round, the player plays a subset of the available arms, subject to feasibility constraints, and collects, for each arm played, the realized reward of the round. Additionally, each of the selected arms becomes unavailable for a random duration equal to the realized delay of the round. The player has access to an oracle that takes as input a subset of arms and a weight vector, and returns an approximately/probably maximum weight feasible subset. The high-level goal is to maximize cumulative reward that is acquired over time.

There are three central challenges towards the player's goal described above. We outline these challenges, and use these as an organizing principle to describe our main contributions.

\noindent{\bf Challenge: Inapproximability.} In the {\em full-information setting}, where we assume prior knowledge of the arm mean rewards, the blocking bandits problem is already NP-hard, even in the simple case where the player can pull at most one arm per round and the arm delays are deterministic \cite{SST09}. In the generic combinatorial bandit setting, the resulting scheduling problem may be {\em NP-hard to even approximate}, even if the underlying feasible set of arms is ``easy'' to optimize. Thus, a fundamental challenge is to identify sufficient conditions on the combinatorial structure of the feasibility constraints under which our problem accepts efficient $\mathcal{O}(1)$-approximation algorithms. 

\noindent{\bf Contribution 1}: We show that if the family of feasible sets satisfies the {\em hereditary property} (subsets of feasible sets are feasible), then our problem accepts $\mathcal{O}(1)$-approximations. 
This is a natural property that appears in many practical applications (e.g., knapsack and matching constraints). Specifically, using a gap-preserving reduction from the {\em Edge-Disjoint Paths} (EDP) problem, we prove that, without the hereditary property, there exists no polynomial-time ${\Omega}(k^{-\frac{1}{2} + \epsilon})$-approximation algorithm for any $\epsilon>0$, where $k$ is the number of arms, unless P = NP. Interestingly, this hardness result holds even when the underlying combinatorial set accepts a polynomial-time algorithm for linear maximization. Finally, we show that even assuming feasible sets that satisfy the hereditary property and admit efficient linear program solution, no efficient algorithm can achieve an approximation ratio greater than $1-\frac{1}{e}+\epsilon$, for any $\epsilon>0$, unless P = NP. We prove this result using a reduction from the {\em Max-k-Cover} problem. Hence, constant approximations are the best one can hope for.

\noindent{\bf Challenge: Analyzing the Greedy Algorithm.} A natural heuristic for the full-information setting can be obtained by greedily playing the maximum expected reward feasible subset among the available arms of each round. Though this greedy step may itself be hard, we follow the paradigm of the combinatorial bandits literature (see \cite{WC17} and references therein) and assume access to a black-box (approximate/randomized) oracle for this (static) greedy problem. Though natural, the performance of such a greedy heuristic in our setting remains unknown. In part, the technical challenge stems from the black-box oracle assumption. Because we have only oracle-access to the greedy algorithm, there is no easy way to characterize the solution returned by the oracle at each round. Further, as opposed to the case of deterministic delays considered in prior work, in our case, any optimal algorithm in hindsight is inherently online and adaptive to random delay realizations. 

\noindent{\bf Contribution 2}: We provide a {\em tight approximation analysis} of the greedy heuristic when the feasible sets of arms satisfy the hereditary property. Specifically, assuming access to an oracle such that, given a weight vector and a set of arms, it returns an $\alpha$-approximation of the maximum weight feasible subset with probability $\beta$, we prove that the greedy heuristic yields an $\frac{\alpha\beta}{1+\alpha \beta}$-approximation (asymptotically) for the full-information setting. The key in proving the above bound is to consider the {\em expected pulling rate} of each arm in an optimal solution. 
Then, by the hereditary property and the fact that the optimal solution is not a priori aware of the delay realizations, we show properties of these rates and, thus, provide the above guarantee.

\noindent{\bf Challenge: Bandit algorithm.} In the case where distributional knowledge of rewards and delays is not assumed, our problem seems significantly harder than standard combinatorial bandits. Even in the absence of blocking, the analysis of bandit algorithms based on the technique of {\em Upper Confidence Bound} (UCB) \cite{ACBF02} becomes highly non-trivial in the combinatorial setting and it required a considerably long line of research to achieve optimal regret guarantees~\cite{WC17,GKJ2012,CWY13,CTPL15,KWAS15a,CWYW16}. These guarantees are obtained by measuring the loss compared to the optimal feasible subset, which is {\em fixed} throughout the time horizon. In the presence of stochastic blocking, these results do not seem to apply, as the optimal solution in hindsight changes dynamically over time according to the set of available arms.

\noindent{\bf Contribution 3}: We develop a UCB-based variant of the greedy heuristic for the bandit case of our problem, where the reward and delay distributions are initially unknown. As already mentioned, the optimal arm-pulling schedule in our case is not fixed, but depends on the randomness of the delays and the player's actions. Our key insight is to control regret by comparing to a static fractional (hence fictitious) solution, rather than to the dynamics of the optimal solution. By leveraging and extending the framework of \cite{WC17} in a way to capture dynamically changing sets of available arms, we are able to provide logarithmic $\frac{\alpha\beta}{1+\alpha\beta}$-approximate regret guarantees. As we note, these guarantees are optimal in the sense that they match the theoretically proven lower bound presented in \cite{KWAS15a} in the limiting case of absence of delays.

\subsection{Related Work}

Following its introduction \cite{T33, LR85}, several variants of the stochastic MAB framework have been thoroughly studied (see \cite{LS20,Sliv19} for an overview). We focus on the case of {\em stochastic combinatorial bandits} that is mostly related to our setting. In this direction, a long line of work~\cite{CWY13,CTPL15,KWAS15a,KWAS15b,CHLLLL16,CWYW16,WC17} focuses on variations of the MAB problem, where more than one arm can be played at each round. The above model has been studied in a high level of generality, including arbitrary feasible sets of arms where linear optimization is achieved via (approximate/randomized) oracles, non-linear reward functions of bounded smoothness, probabilistically triggered arms and more. Starting with the work of Gai et al. \yrcite{GKJ2012}, the case of arbitrary feasible sets and linear reward functions has been particularly studied. Given $k$ arms and a time horizon $T$, the state-of-the-art \cite{WC17,KWAS15a} in that case is a UCB-based bandit policy of $\mathcal{O}(\frac{k\cdot r}{\Delta} \log(T))$ regret against the best possible solution that is achievable in polynomial time, where $r$ is the maximum cardinality of a feasible set and $\Delta$ is the gap in expected reward between the optimal and the best suboptimal feasible set. As proved in \cite{KWAS15a}, the above regret bound matches the theoretical lower bound for this setting.

Our model falls into the area of {\em stochastic non-stationary bandits}. Important lines in this area include {\em restless bandits}, where the arms' mean rewards change at every round \cite{Whittle88, GMS10}, and {\em rested bandits}, where the means can change only when the arm is played \cite{Gittins79,TL12}. The above classes appear to contain notoriously hard problems from a computational viewpoint. In fact, even approximating the optimal solution for the class of restless bandits is PSPACE-hard \cite{PT1999}. Moreover, our model also belongs to the class of {\em Markov Decision Processes} (MDPs) \cite{P94}. However, modeling our problem as an MDP would be inefficient, as it would require a state space that grows exponentially in the number of arms.

Recently, there has been much interest in several variants of the non-stationary model where reward distributions exhibit special temporal correlations with the player's past actions \cite{KI18,PBG19,CCB19,LKLM20}. In \cite{BSSS19}, the authors first study the blocking bandits problem, where each arm becomes blocked for a deterministic number of time steps after it has been played. The above problem has also been studied in an adversarial setting \cite{BCMT20} and in a contextual setting \cite{BPCS20}, where the mean rewards of the arms depend on a stochastic context that is observed by the player at the beginning of each round. Very recently, \cite{PC21} generalize the problem to the setting where more than one arms can be played at each round, subject to matroid constraints. However, none of these studies handle either stochastic delays, or a general class of independence systems, both of which are important in practice (e.g., knapsack problems cannot be handled by any prior studies). Furthermore, in terms of techniques, the analysis of \cite{PC21} (the only one above addressing combinatorial constraints) heavily relies on the {\em submodularity} of the {\em matroid rank function} and the fact that the delays are known and deterministic. We require neither of these to hold in our setting.

\section{Preliminaries}
We consider a set of $k$ arms, denoted by $\A$, and an {\em unknown} time horizon $T \in \mathbb{N}$. Each arm $i \in \A$ is associated with a reward distribution $\mathcal{X}_i$ that is bounded w.l.o.g. in $[0,1]$. We denote by $\mu_i$ the mean reward of arm $i$. In the blocking setting, whenever an arm $i$ is pulled at some round $t$, it cannot be played again for $D_{i,t}-1$ subsequent rounds, namely, within the interval $\{t,...,t+D_{i,t}-1\}$ (i.e., $D_{i,t} = 1$ implies that the arm is not blocked). 
For each arm $i \in \A$, the value $D_{i,t} \in \mathbb{N}_{\geq 1}$, which we refer to as the {\em delay}, is a random variable drawn from some arm-dependent distribution $\mathcal{D}_i$ of mean $d_i = \mathbb{E}[D_{i,t}], \forall t \in [T]$ and bounded support in $[1, d^{\max}_i]$. At each round $t$ and for each arm $i \in A$, the reward and delay realization, $X_{i,t}$ and $D_{i,t}$, respectively, are drawn independently from the joint distribution $\mathcal{X}_i, \mathcal{D}_i$ and, thus, they are allowed to be correlated. We denote by $d_{\max} = \max_{i \in \A}\{d^{\max}_i\}$ the maximum delay in a given instance. 

We consider the setting of combinatorial bandits, where more than one arm can be played at each round, subject to feasibility constraints. Let $\I\subseteq \{0,1\}^k$ be the family of {\em feasible} subsets of $\A$. For any subset of arms $S \subseteq \A$, we denote by $\I(S) = \{S' \in \I~|~S' \subseteq S\}$ the subset of feasible sets that only contain arms from $S \subseteq \A$. Let $r$ be the maximum cardinality of a set in $\I$, namely, $r = \max_{S \in \I}\{|S|\}$. Since the problem of maximizing a linear function over a feasible family $\I$ can be NP-hard, following the paradigm of combinatorial bandits as in \cite{WC17, CWYW16}, access to the feasible set is given through an oracle. Given a non-negative weight vector $\mu \in \mathbb{R}_{\geq 0}^k$ and a set $S \subseteq \A$, let $\OPT_{\mu}(S)$ be the maximum weight feasible set in $\I(S)$ w.r.t. $\mu$. For any $S\subseteq \A$ and vector $\mu \in \mathbb{R}^k_{\geq 0}$, we use the notation $\mu(S)=\sum_{i\in S}\mu_i$. For $\alpha, \beta \in (0,1]$, we assume access to a polynomial-time $(\alpha, \beta)$-approximation oracle, $f_{\mu}^{\alpha, \beta}$, for the underlying combinatorial problem, defined as follows:

\begin{definition}[($\alpha,\beta$)-approximation oracle] Given a weight vector $\mu \in \mathbb{R}^k$ and a subset $S \subseteq \A$, an $(\alpha, \beta)$-approximation oracle for $\alpha, \beta \in (0,1]$ outputs a set $f_{\mu}^{\alpha,\beta}(S) \in \I(S)$, such that 
$$\Pro{ \mu(f_{\mu}^{\alpha,\beta}(S)) \geq \alpha \cdot \mu(\OPT_{\mu}(S))} \geq \beta.$$
\end{definition}

We are now ready to describe the setting: At each round $t$, after observing the set of available arms (that is, the arms that are not blocked by some previous pulling), the {\em player} plays any feasible subset $\A_t$ of these arms, namely, $\A_t \in \I$ and collects the realized rewards. At that point, each arm $i \in \A_t$ becomes blocked with delay $D_{i,t}$. We emphasize that the player is {\em initially unaware of the delay distributions} and can only infer each delay realization $D_{i,t}$ by observing the time where arm $i$ becomes available again. The player seeks to maximize her {\em total expected reward} in $T$ rounds, formally defined as:
$$
    \Rew^{\pi}(T) = \Ex{}{\sum_{t\in [T]} \sum_{i\in \A} X_{i,t} \event{i\in A_t^{\pi}}},
$$
where $A_t^{\pi}$ denotes the subset of arms played by an algorithm $\pi$ at round $t$. We refer to the above setting as {\em Combinatorial Blocking Bandits with Stochastic Delays} (\CBBSD).

In the bandit setting, we compare the performance of our algorithm with the expected reward of an optimal algorithm that has distributional knowledge of the arm rewards and delays, is aware of the time horizon $T$ and has infinite computational power. Let $\Rew^{*}(T)$ be the optimal expected reward. To evaluate our algorithm we compute its approximate regret compared to the optimal. The following definition of $\rho$-{\em approximate regret} (or $\rho$-regret) is standard in the field of combinatorial bandits for underlying feasibile sets where linear maximization is NP-hard:
\begin{align*}
    \Reg_{\rho}^{\pi}(T) &:= \rho \Rew^{*}(T) - \Rew^{\pi}(T).
\end{align*}

We present two important conditions and investigate their necessity for proving the approximation guarantee of our proposed algorithm. 
In the standard MAB framework, it is assumed that the optimal algorithm only knows the distributions of the rewards, but not the realizations. The following condition states that the optimal algorithm is also unaware of the delay realizations before pulling an arm:

\begin{condition}\label{assumption:delays}
The optimal online algorithm has knowledge of the delay distributions, but is not a priori aware of the delay realizations.
\end{condition}

A family $\I$ of feasible subsets is called an {\em independence system}, if it satisfies the following property:
\begin{definition}[Hereditary Property] \label{def:herediraty}
    Every subset of a feasible set in $\I$ is a also a feasible set, that is, $S' \subset S \subseteq \A$ and $S \in \I$ implies that $S' \in \I$.
\end{definition}

We introduce the following condition on the family of feasible sets of our problem:

\begin{condition}\label{assumption:independence}
We assume that the family of feasible set $\I$ in every instance of our problem is an independence system.
\end{condition}

The above two conditions are critical. Indeed, as we show in Section \ref{sec:hardness},  dropping any of these makes the problem intractable, even in an approximate sense.

\section{Computational-Unconditional Hardness} \label{sec:hardness}

In this section, we study the hardness of the \CBBSD~problem. We emphasize that all the results of this section hold even in the simple case where the arm rewards are deterministic and known to the player.

We first show that even in the simple case of a single arm that is always feasible to play (if available), we cannot compete withing a $\omega({1}/{d_{\max}})$-factor against an optimal solution that knows the delay realizations.

\begin{restatable}{theorem}{hardnessDelays}\label{thm:hardness:delays}
In the case where \Cref{assumption:delays} does not hold (that is, the optimal algorithm knows the delay realizations of all rounds), there exists no algorithm (not even one of infinite computational power) for the \CBBSD~problem that achieves an approximation ratio of $\omega(\frac{1}{d_{\max}})$.
\end{restatable}
{\em Proof sketch.}
We consider the case of an infinite time horizon and a single arm ($k=1$) of deterministic reward, equal to $1$. The delay of the arm is either $d > 1$ or $1$, each with probability $\frac{1}{2}$. The proof of the theorem follows by modeling the optimal policy in the above example as a Markov Chain, and comparing its expected average reward with that of a possibly sub-optimal policy, which is a priori aware of the delay realizations and plays the arm only when the realized delay is $1$.
\qed

We now focus on the hardness of the problem, restricting our attention to the simpler case where the arm rewards and delays are deterministic and known. In the next result, we show that if the underlying feasible set $\I$ does not satisfy \Cref{assumption:independence}, then, for any $\epsilon > 0$, there cannot exist any polynomial-time algorithm of approximation guarantee better than ${\Omega}({k^{-\frac{1}{2} + \epsilon}})$, where $k$ is the number of arms, unless P = NP. Interestingly, the above result holds even if linear maximization over the family of feasible sets $\I$ can be done efficiently at every round (i.e., having access to a $(1,1)$-approximation oracle). We prove the above claim by a gap-preserving reduction from the following problem:

\begin{definition}[Edge-disjoint paths (EDP)]
Given a directed graph $\mathcal{G}=(V,E)$, where $V$ is the set of vertices and $E$ is the set of $k'$ edges, and $m$ pairs of vertices $\mathcal{T}=\{(s_i,t_i)~|~s_i,t_i\in V, i\in [m]\}$, compute the maximum number of $(s_i,t_i)$ pairs that can be connected using edge-disjoint paths.
\end{definition}

As we show, the full-information case of our bandit problem captures the hardness of EDP, if we do not assume that the underlying feasible set of arms in $\I$ is an independence system (that is, satisfying \Cref{assumption:independence}). The following result is known for the EDP problem.

\begin{theorem}[\cite{GKRSY99}] \label{thm:edp:gap_problem}
For any $\epsilon>0$, given a directed graph $\mathcal{G}=(V,E)$ with $k' =|E|$ and a set of $m$ pairs of vertices $\mathcal{T}=\{(s_i,t_i)~|~s_i,t_i\in V, i\in [m]\}$, it is NP-hard to distinguish whether all pairs in $\mathcal{T}$ or at most a fraction of $\frac{1}{k'^{{1}/{2}-\epsilon}}$ of the pairs can be connected by edge-disjoint paths.
\end{theorem}

We now prove the following theorem:

\begin{restatable}{theorem}{hardnessPoly}
\label{thm:hardness:poly}
Unless P = NP and for any $\epsilon > 0$, there exists no polynomial-time $\Omega(k^{-\frac{1}{2}+\epsilon})$-approximation algorithm for \CBBSD~problem, in the case of feasible sets that are not independence systems. The above holds even for feasible sets where linear optimization can be performed efficiently.
\end{restatable}

{\em Proof sketch.}
The key idea is to construct an instance of \CBBSD, where each arm represents an edge in a given directed graph. In this instance, a subset of arms is feasible only if it includes a unique path between some $(s_i, t_i)$ pair. Our instance is constructed in a way that any feasible solution yields a reward equal to $1$ (using linear rewards on the arms). Further, finding a feasible solution in a given subset of arms/edges can be achieved in polynomial-time via breadth/depth first search. As we show, the existence of a polynomial-time $\Omega(k^{-\frac{1}{2}+\epsilon})$-approximation algorithm for \CBBSD~problem would allow us to resolve the gap problem of \Cref{thm:edp:gap_problem}. \qed

Finally, we focus on the hardness of the \CBBSD~problem, in the case of feasible sets that are independence systems (i.e., satisfying \Cref{assumption:independence}). As we show, although under this assumption the problem is easier, yet it still exhibits $\Omega(1)$-hardness of approximation. Specifically, we prove the hardness of \CBBSD~in that case, using a reduction from Max-k-Cover:

\begin{definition}[Max-k-Cover]
Given a universe $U = \{e_1, \dots, e_k\}$ of $k$ of elements and a collection of $m$ subsets $S_1, \dots, S_m \subseteq U$, choose $l$ sets that maximize the number of covered elements.
\end{definition}
The following hardness result is known for Max-k-Cover:

\begin{theorem}[\cite{Feige98}] \label{thm:feige} Unless P = NP, there is no polynomial-time $1 - \frac{1}{e} + \epsilon$-approximation algorithm for Max-k-Cover, for any $\epsilon >0$.
\end{theorem}

Using the above result, we are able to show the following hardness result for our problem:

\begin{restatable}{theorem}{hardnessDown}
\label{thm:hardness:down}
Unless P = NP and for any $\epsilon > 0$, there exists no polynomial-time $(1 - \frac{1}{e} + \epsilon)$-approximation algorithm for the full-information \CBBSD~problem, even for feasible sets that satisfy the hereditary property. 
The above holds even for feasible sets where linear optimization can be done efficiently. 
\end{restatable}
{\em Proof sketch.}
For any instance of Max-k-Cover, we can construct in polynomial-time an instance of the \CBBSD~problem by considering an arm of reward $1$ and delay $l$ for each element of $U$, where $l$ is the number of sets that can be chosen in Max-k-Cover. Further, a subset of arms $S$ is feasible, only if $S \subseteq S_j$ for at least one given subset $S_j \subseteq U$ (notice that this construction satisfies the hereditary property). As we show, any $\rho$-approximation algorithm for the above instance of \CBBSD, would imply a $\rho$-approximation for Max-k-Cover, which, in turn, implies the hardness of the former.
\qed
\section{Full-Information Setting} \label{sec:fullinformation}

We first study the full-information setting of our problem, where we assume that the player has {\em prior knowledge of the arm mean rewards}. In the case where Conditions \ref{assumption:delays} and \ref{assumption:independence} hold, we analyze the approximation guarantee of the following greedy heuristic: 

\textsc{greedy-heuristic:} At each time $t = 1,2,\dots$, observe the set $\F_t \subseteq \A$ of available (not blocked) arms. Let $\A_t = f_{\mu}^{(\alpha,\beta)}(\F_t)$ be the subset returned by the $(\alpha, \beta)$-oracle for support $\F_t$ and weight function $\mu$. Play the arms of $\A_t$ and collect the associated rewards.

We remark there are two sources of randomness: arm rewards and stochastic delays. In the full information setting, we assume that the player knows the mean rewards of the arms, beforehand. We drop this assumption in the bandit setting, studied in \Cref{sec:bandit}. Moreover, we note that the algorithm provided above does not make use, and thus does not require knowledge of the delay distributions (or realizations). 
As we prove in the rest of this section, the above algorithm is $\rho(\alpha, \beta)$-competitive in expectation, against an optimal online algorithm, for $\rho(\alpha, \beta) = \frac{\alpha \beta}{1 + \alpha \beta}$.

\begin{theorem} \label{thm:full_information}
Given any $(\alpha, \beta)$-oracle and assuming \Cref{assumption:delays} and \Cref{assumption:independence}, the algorithm \textsc{greedy-heuristic} is a $\frac{\alpha \beta}{1 + \alpha \beta}$-approximation (asymptotically) for the full-information case of \CBBSD.
\end{theorem}

The rest of this section is dedicated to proving the above result. For an algorithm $\pi$, we denote as $A_t^{\pi}$ the set of arms played by $\pi$ at time $t$ and $\F_t^{\pi} \subseteq \A$ (resp., $\B_t^{\pi} \subseteq \A$) the set of \textit{available} (resp., \textit{blocked}) arms at the beginning of the round. Finally, for any $S\subseteq \A$ and vector $\mu \in \mathbb{R}^k_{\geq 0}$, we use the notation $\mu(S)=\sum_{i\in S}\mu_i$.
Let $\A^*_t$ be the set of arms played by an optimal algorithm at time $t \in [T]$. For proving the performance guarantee of our algorithm in the presence of stochastic delays, the first step is to consider the {\em expected pulling rate} of each arm $i \in \A$ by the optimal algorithm, defined as: $z_i = \Ex{}{\frac{1}{T}\sum_{t\in [T]}\event{i \in \A_t^*}}$. 

Given the fact that the optimal algorithm is not aware of $\{X_{i,t}\}_{i \in \A}$ before deciding which arms to pull at time $t$, its expected cumulative reward is:
\begin{align*}
\Ex{}{\sum_{t \in [T]} \sum_{i \in \A}  X_{i,t} \event{i \in \A^*_t}}
= \sum_{t \in [T]} \sum_{i \in \A} \mu_i \Ex{}{{\event{i \in \A^*_t}}},
\end{align*}
where the equality follows by the fact that $X_{i,t}$ and $\event{i \in \A^*_t}$ are independent. Thus, we obtain that:
\begin{align}
\Rew^*(T)= T \sum_{i\in A}\mu_i z_i.    \label{claim:objective}
\end{align}
We now prove two important properties of the expected pulling-rates $\{z_i\}_{i \in \A}$, each following from \Cref{assumption:delays} and \Cref{assumption:independence}, respectively. In the following lemma, we provide an upper bound on $z_i$, for each arm $i \in \A$.

\begin{restatable}{lemma}{stochastic} \label{lemma:stochastic} 
Let $i \in \A$ be an arm of expected delay $d_i = E[D_{i,t}]$ and maximum delay $d^{\max}_i$. We have:
$$
z_i \leq \frac{1}{d_i}+ \mathcal{O}\Big(\frac{d_i^{\max}}{T}\Big).
$$
\end{restatable}

For proving the above claim, we make crucial use of \Cref{assumption:delays}. Indeed, it is not hard to see that in the construction of \Cref{thm:hardness:delays}, the above inequality does not hold.

\begin{lemma} \label{lemma:hereditary}
For any subset $S\subseteq \A$ and reward vector $\mu$, assuming that $\I$ is an independence system, we have:
$$
\mu(OPT_{\mu}(S)) \geq \sum_{i\in S}\mu_i z_i.
$$
\end{lemma}
\begin{proof}
Let $\A^*_t$ be the set of arms played by the optimal algorithm at time $t$. Since for any set $S$ and time $t$ the intersection $S \cap \A^*_t$ is strictly contained in $S$ and $\A^*_t$, by \Cref{assumption:independence} we have that $S \cap \A^*_t \in \I(S)$. Therefore, for the expected reward of the optimal solution in $\I(S)$, we have:
\begin{equation}
    \mu(\OPT_{\mu}(S)) \geq \mu(S \cap \A_t^*) = \sum_{i\in S} \mu_i {\event{i\in \A_t^*}}. 
\end{equation}
By averaging over time and taking the expectation over the randomness of the delays and the possible random bits of the optimal algorithm, we can conclude that:
\begin{equation*}
    \mu(\OPT_{\mu}(S)) \geq \frac{1}{T} \sum_{t\in [T]} \sum_{i\in S} \mu_i \Pro{i\in \A_t^*} =  \sum_{i\in S} \mu_i z_i.
\end{equation*}
\end{proof}

Equipped with the above lemmas, we can now complete the approximation analysis of \textsc{greedy-heuristic}.

\textit{Proof of \Cref{thm:full_information}.} 
Using the fact that the choices of \textsc{greedy-heuristic} do not depend on the reward realizations, for each round $t \in [T]$ we have:
\begin{align*}
\sum_{i \in \A} \Ex{}{X_{i,t} \event{i \in \A_t} } = \sum_{i \in \A} \Ex{}{\mu_i \event{i \in \A_t}} = \Ex{}{\mu(\A_t)}.
\end{align*}

We denote by $\F^{\pi}_t$ and $\B_t^{\pi}$ the set of available and blocked arms at the beginning of round $t$, respectively. Let $\mathcal{Q}_t$ denote the event that the oracle succeeds in returning an $\alpha$-approximate solution at time $t$. By taking expectation over the randomness of the oracle and the delay realizations, for every $t \in [T]$, we have that:
\begin{align*}
\Ex{}{\mu(\A_t)} &\geq \Ex{}{\alpha \cdot \mu(\OPT_{\mu}(\F^{\pi}_t)) \event{\mathcal{Q}_t}} \\
    &\geq \alpha\cdot\beta\cdot \Ex{}{\mu(\OPT_{\mu}(\F^{\pi}_t)}) \\
    &\geq \alpha\cdot\beta\cdot \Ex{}{\sum_{i \in \F^{\pi}_t}\mu_i z_i} ,
\end{align*}
where the first inequality follows by definition of the oracle, the second by the fact that $\mathcal{Q}_t$ is independent of the set $\F^{\pi}_t$ and that $\Pro{\mathcal{Q}_t} \geq \beta$, and the third by \Cref{lemma:hereditary}. Thus, using \Cref{claim:objective}, we have:
\begin{align}
&\sum_{t \in [T]}\Ex{}{\sum_{i \in \A} X_{i,t} \event{i \in \A_t}} \nonumber\\
&\geq \alpha\cdot\beta \sum_{t \in [T]} \Ex{}{\sum_{i \in \F^{\pi}_t}\mu_i z_i} \nonumber\\
&= \alpha \cdot \beta \cdot T \sum_{i \in \A} \mu_i z_i - \alpha\cdot\beta \sum_{t \in [T]}\Ex{}{\sum_{i \in \B^{\pi}_{t}} \mu_i z_i} \nonumber\\
&= \alpha \cdot \beta  \Rew^{*}(T) - \alpha\cdot\beta \sum_{t \in [T]}\Ex{}{\sum_{i \in \B^{\pi}_{t}} \mu_i z_i}, \label{inq:online:first}
\end{align}
where, for the first equality, we use that at any round $t \in [T]$, we have $\F^{\pi}_t \cup \B^{\pi}_t = \A^{\pi}_t$ and $\F^{\pi}_t \cap \B^{\pi}_t=\emptyset$.

Note that for any time $t \in [T]$ and arm $i \in \B^{\pi}_t$, we have:
$$\event{i \in B^{\pi}_t} = \sum_{t' < t} \event{i \in \A^{\pi}_{t'}} \event{D_{i,t'} > t - t'}.$$
Indeed, in order for an arm to be blocked at time $t$, it must be played at some earlier round $t' < t$, and the realized delay at $t'$ must be greater than $t - t'$. Using this fact, we have:
\begin{align}
    &\Ex{}{\sum_{t\in [T]} \sum_{i\in \B_t^{\pi}} \mu_i z_i} \nonumber\\
    &=\Ex{}{\sum_{t\in [T]} \sum_{i\in \A} \mu_i z_i \sum_{t' < t} \event{i \in \A^{\pi}_{t'}} \event{D_{i,t'} > t - t'}} \nonumber\\
    &= \sum_{t\in [T]} \sum_{i\in \A} \mu_i z_i \sum_{t' < t}\Pro{i \in \A_{t'}^{\pi}}\Pro{D_{i,t'}>t-t'} \nonumber\\
    &= \sum_{t'\in [T]}
   \sum_{i\in \A} 
    \Pro{i \in \A_{t'}^{\pi}} \mu_i z_i \sum_{t > t'}\Pro{D_{i}>t-t'} \nonumber\\
    &\leq \sum_{t'\in [T]}   \sum_{i\in \A}
    \Pro{i \in \A_{t'}^{\pi}} \mu_i z_i d_i \nonumber\\
    &\leq \Rew^{\pi}(T) + \mathcal{O}\left({d_{\max} \cdot k}\right), \label{inq:online:second}
\end{align}
where the second equality holds since the realization $D_{i,t'}$ is independent of the event $i \in A^{\pi}_{t'}$. The first inequality is due to the fact that $D_{i,t'}$ is a non-negative integer random variable, thus, $d_i = \Ex{}{D_{i,t'}} = \sum^{\infty}_{\tau = 1} \Pro{D_{i,t'} \geq \tau}$. Finally, the last inequality is due to \Cref{lemma:stochastic}. 
By combining Inequalities \eqref{inq:online:first} and \eqref{inq:online:second}, we get:
\begin{align*}
    \Rew^{\pi}(T) \geq \frac{\alpha\cdot \beta}{1 + \alpha\cdot \beta} \Rew^{*}(T) - \mathcal{O}\left({d_{\max} \cdot k}\right). \qed
\end{align*}

The above analysis is tight, since there exists an instance of the \CBBSD~problem where \textsc{greedy-heuristic} collects an exact $\frac{\alpha\cdot \beta}{1 + \alpha\cdot \beta}$-fraction of the optimal reward (see Appendix E in \cite{PC21}). 
\section{Bandit Setting} \label{sec:bandit}
We now turn our attention to the bandit setting where the player is initially unaware of the mean rewards. Given an $(\alpha, \beta)$-oracle for the feasible set of arms and assuming Conditions \ref{assumption:delays} and \ref{assumption:independence}, we develop a UCB-based variant of \textsc{greedy-heuristic}, for which we prove $\rho(\alpha, \beta)$-regret guarantees.

\paragraph{Bandit algorithm.}
Let us denote by $T_{i,t}$ the number of times arm $i$ has been played up to and including round $t \in [T]$, and by ${\muh}_{i,t}$ the empirical average of $T_{i,t}$ independent samples from the distribution $\mathcal{X}_i$. We design a UCB-based variant of \textsc{greedy-heuristic}, which we call \textsc{cbbsd-ucb}, that maintains at each time $t$ the following estimates (UCB-indices) of the mean rewards:
\begin{align*}
\mub_{i,t} = \min\Big\{ \muh_{i,t-1} + \sqrt{\frac{3 \ln t}{2T_{i,t-1}}},1\Big\}, \forall t \in \A, t \in [T].    
\end{align*}
In the following, we use $\pit$ for any reference to \textsc{cbbsd-ucb}. 

The \textsc{cbbsd-ucb} algorithm ( \Cref{alg:CUCB}) works as follows: Let $\F^{\pit}_t$ denote the set of available arms. At each time $t \in [T]$ the algorithm observes $\F^{\pit}_t$, and  computes and plays the feasible set of arms $f_{\mub_t}^{\alpha,\beta}(\F^{\pit}_t)$ returned by the oracle, where $\mub_t \in \mathbb{R}^k_{\geq 0}$ is the vector of the UCB-indices, such that $(\mub_t)_i = \mub_{i,t}$. Then, it observes and collects the reward realizations of the played arms and updates the UCB-indices.

\begin{algorithm}[tb]
  \caption{CBBSD-UCB}
  \label{alg:CUCB}
\begin{algorithmic}
  \STATE {\bfseries Input:} Set of arms $A$ and oracle $f^{\alpha,\beta}$.
  \STATE For every arm $i\in \A$, set $T_i\leftarrow 0$ and $\muh_i\leftarrow 1$.
  \FOR{t=1,2,...}
  \STATE For every arm $i\in \A$, $\mub_{i} \leftarrow \min\{ \muh_{i} + \sqrt{\frac{3 \ln t}{2T_{i}}},1\}$
  \STATE Play the set $\A^{\pit} \leftarrow f_{\mub}^{\alpha,\beta}(\F^{\pit})$.
  \STATE Observe the realization ${X}_{i}$ of $\mathcal{X}_{i}$, for $i\in \A^{\pit}$. 
  \STATE Set $T_i \leftarrow  T_i+1$ and $\muh_i \leftarrow \muh_i\frac{T_i-1}{T_i} + \frac{X_i}{T_i}$, for $i\in \A^{\pit}$.
  \ENDFOR
\end{algorithmic}
\end{algorithm}

\paragraph{Regret analysis.}
We now provide $\rho(\alpha,\beta)$-regret guarantees for algorithm \textsc{cbbsd-ucb}, with $\rho = \rho(\alpha, \beta) = \frac{\alpha\beta}{1 + \alpha\beta}$. As opposed to the non-blocking combinatorial bandits setting, where the optimal solution in hindsight is to repeatedly play the maximum expected reward feasible subset of arms, the mean reward collected at each round by an optimal solution in the blocking setting might exhibit significant fluctuations over time due to blocking. Standard regret analysis typically uses the optimal solution's value as a baseline; but these fluctuations require a different idea.

Instead, we use as a baseline the expected reward collected by an optimal full-information algorithm that can play arms {\em fractionally} yet {\em consistently} over time. Indeed, the (optimal) expected pulling rates, $\{z_i\}_{i \in \A}$, as defined in \Cref{sec:fullinformation}, precisely characterize such a baseline. These rates can be used to mirror the analysis of \Cref{thm:full_information}, this time for \textsc{cbbsd-ucb}, keeping track of the extra loss due to the lack of information. 
In order to quantify the above loss, which is due to the fact that \textsc{cbbsd-ucb} calls the $(\alpha,\beta)$-oracle using the vector of UCB-indices $\mub_t$ at each time $t$, we use the notion of {\em instantaneous regret}:
\begin{definition}[Instantaneous Regret]
At every time $t \in [T]$, the {\em instantaneous regret} of $\pit$ is defined as:
\begin{equation*}
    \Gamma^{\pit}_{t} = \alpha \cdot \beta \cdot \mu(\OPT_{\mu}(\F^{\pit}_t)) - \mu(\A^{\pit}_t).
\end{equation*}
\end{definition}
Informally, $\Gamma^{\pit}_{t}$ measures the difference between an $\alpha \beta$-fraction of the expected reward of the optimal feasible subset of the available arms, and the expected reward collected by $\pit$ at time $t$.

Using the properties of the expected pulling rates in \Cref{claim:objective} and Lemmas \ref{lemma:stochastic} and \ref{lemma:hereditary}, and following \Cref{thm:full_information}, we provide the following upper bound to the regret:

\begin{restatable}{lemma}{gammaDecomp}\label{lemma:gamma}
The $\rho$-approximate regret of our algorithm, where $\rho = \frac{\alpha\cdot\beta}{1 + \alpha\cdot\beta}$ can be upper bounded as:
\begin{align*}
    &\rho\Reg^{\pit}(T) \leq \frac{1}{1 + \alpha\cdot \beta} \Ex{}{\sum_{t \in [T]} \Gamma^{\pit}_t} + \mathcal{O}(d_{\max}\cdot k). 
\end{align*}
\end{restatable}

Thus in order to upper bound the $\rho$-regret, we need to control the instantaneous regret over time. While this task now resembles the regret analysis of standard (non-blocking) combinatorial bandits, our setting poses an additional technical challenge: the instantaneous regret $\Gamma^{\pit}_{t}$ depends on the history of arm pulling and reward/delay realizations, not only via the state of the UCB-indices at time $t$, but also through the set of available arms $\F^{\pit}_t$. As we show, any potential issue due to these additional correlations can be avoided by carefully defining the suboptimality gaps of our problem and adapting the techniques in \cite{WC17} in a way to capture the constantly changing availability state.

\begin{definition}[Bad Feasible Set]\label{def:bad_set}
We refer to a feasible set $S \in \I(\Sc)$ as {\em bad} w.r.t. a availability set $\Sc$, when $\mu(S) < \alpha \cdot \mu(\OPT_{\mu}(\Sc))$. Moreover, the family of bad feasible sets w.r.t. the availability set $\Sc$ is defined as: 
$$\mathcal{S}_B(\Sc)= \{S \in \I(\Sc) ~|~ \mu(S) < \alpha \cdot \mu(\OPT_{\mu}(\Sc))\}.$$
Finally, we denote by $\mathcal{S}_{i,B}(\Sc)$ the family of feasible sets in $\mathcal{S}_B(\Sc)$ that contain arm $i \in \A$.
\end{definition}

As opposed to \cite{CWY13, WC17}, the following notion of suboptimality gap is now a function of both the availability set $\F$, which depends on the choices of the algorithm, and the played set $S \in \I(\F)$. 

\begin{definition}[Suboptimality Gap]\label{def:differences}
The {\em suboptimality gap} of a bad feasible set $S \in \I(\F)$ w.r.t. an availability set $\Sc$ is defined as: 
$$\Delta(S,\Sc)= \alpha \cdot \mu(\OPT_{\mu}(\Sc)) - \mu(S).$$
\end{definition}

Focusing on the bad feasible sets that contain arm $i$, we can further define the minimum possible gap of any such set as:
$\Delta^{i}_{min} = \min_{\Sc \subseteq \A,~ S\in S_{i,B}(\Sc)}\Delta(S,\Sc).$
Further, we can define maximum gap of any bad feasible set that can be played as:
$\Delta_{max} = \max_{\Sc \subseteq \A,~i\in \A,~S\in S_{i,B}(\Sc)}\Delta(S,\Sc)$. We remark that the above choice of $\Delta^{i}_{min}$ captures the following crucial aspect of our problem: the algorithm needs to be able to distinguish between optimal and suboptimal choices over any possible availability set $\F$ in the worst case.

For the rest of our analysis, we adapt the ideas introduced in \cite{WC17} to accommodate stochastic blocking. The key technical challenge we circumvent stems from the dynamic (un)availability of the arms, which impacts our ability to obtain accurate estimates of $\mu$. This becomes important in our proofs of Lemma \ref{lemma:kappas} and Theorem \ref{theorem:regret}. 

\begin{definition}[Nice sampling]
At each the beginning of each round $t$, we say that \textsc{cbbsd-ucb} has a {\em nice sampling} if 
$|\muh_{i,t-1}-\mu_i|\leq \sqrt{\frac{3\ln{(t)}}{2T_{i,t-1}}}$ for each arm $i \in \A$.
\end{definition}
Let $\mathcal{N}_t$ denote the event that the algorithm has a nice sampling at time $t$. We can bound the probability of the event $\mathcal{N}_t$ as follows:

\begin{restatable}{lemma}{niceRun} \label{lemma:nice_run}
The event $\mathcal{N}_t$ holds for \textsc{cbbsd-ucb} at round $t$ with probability at least $1-\frac{2k}{t^2}$.
\end{restatable}

Intuitively, the event $\mathcal{N}_t$ implies that the UCB-indices of the arms at time $t$ are a good approximation of the actual mean rewards. At each round $t$, there are three reasons why our algorithm might play a suboptimal set of arms: (i) Failure of the oracle, denoted by $\neg \mathcal{Q}_t$, (ii) Non-representative collection of samples, denoted by $\neg \mathcal{N}_t$, and (iii) Insufficient number of samples for distinguishing the gaps. The probabilities of $\neg \mathcal{Q}_t$ and $\neg \mathcal{N}_t$ at each round $t$ can be upper bounded by $1-\beta$, using the definition of the oracle, and by $\frac{2k}{t^2}$, using \Cref{lemma:nice_run}, respectively. Thus, we focus on the rounds where $\mathcal{Q}_t$ and $\mathcal{N}_t$ hold. 

\begin{definition}[Sampling Threshold]
We define the {\em sampling threshold} as:
\begin{equation*}
    \ell_T(\Delta)= \frac{24r^2\ln(T)}{\Delta^2},
\end{equation*}
\end{definition}
where $r = \max_{S \in \I}\{|S|\}$.
We also define the following function:
\begin{align*}
\kappa_T(\Delta, s) = 
\begin{cases}
       2, &\quad\text{if }s = 0,\\
       \sqrt{\frac{24 \ln(T)}{s}}, &\quad\text{if } 1 \leq s \leq \ell_{T}(\Delta),\\
       0, &\quad\text{if } s > \ell_{T}(\Delta).\\
    \end{cases}
\end{align*}
In the above definitions, which appear in the framework of \cite{WC17}, the domain of $\Delta$ is over any possible suboptimality gap given by a combination of availability set $\F$ and feasible set $S \in \I(\F)$. Specifically, the sampling threshold $\ell_T$ has the following property: When the number of times an arm $i$ has been played exceeds $\ell_T(\Delta_{\min}^i)$, then the bad feasible sets containing this arm cannot further contribute to the regret (assuming that $\mathcal{Q}_t$ and $\mathcal{N}_t$ hold). This idea is depicted in the following result:

\begin{restatable}{lemma} {sumOfKappa}\label{lemma:kappas}
For any $t\in [T]$, if the event $\{\mathcal{Q}_t, \mathcal{N}_t, \A^{\pit}_t\in \mathcal{S}_B(\F^{\pit}_t)\}$ holds, then 
$$\Delta(\A^{\pit}_t,\F^{\pit}_t)\leq \sum_{i\in\A^{\pit}_t} \kappa_T(\Delta_{\min}^i,T_{i,t-1}).$$
\end{restatable}

By combining \Cref{lemma:gamma,lemma:nice_run,lemma:kappas} with the techniques of \cite{WC17}, we obtain the following regret bounds:

\begin{restatable}{theorem}{thmRegret} \label{theorem:regret}
For the $\frac{\alpha\cdot\beta}{1 + \alpha\cdot\beta}$-approximate regret of our algorithm, we provide the following {\em distribution-dependent} bound:
\begin{align*}
    \frac{48}{1 + \alpha\beta}\sum_{i\in \A}\frac{r\cdot \ln{T}}{\Delta_{\min}^i} + k \cdot (2+\frac{\pi^2}{3} \Delta_{\max})  + \mathcal{O}(d_{\max}\cdot k),
\end{align*}
and the following {\em distribution-independent} bound:
\begin{align*}
    \frac{14\sqrt{k\cdot r\cdot T\ln{T}}}{1 + \alpha \beta} + k \cdot (2+ \frac{\pi^2 }{3} \Delta_{\max})  + \mathcal{O}(d_{\max}\cdot k),
\end{align*}
where $r = \max_{S \in \I}\{|S|\}$.
\end{restatable}

Note that except for the additive $\mathcal{O}(d_{\max} \cdot k)$ term, the above regret bounds match the theoretical lower bounds for combinatorial bandits setting presented in \cite{KWAS15a} in the absence of delays. Indeed, given that our problem strictly generalizes the setting of standard combinatorial bandits with linear rewards and without blocking constraints (in which case our regret is exact and the $\mathcal{O}(d_{\max} \cdot k)$ term disappears), our dependence on $k,r,T$ and $\Delta_{\min}$ is unimprovable.

\section*{Acknowledgements}

This research was partially supported by NSF Grant 1826320, ARO grant W911NF-17-1-0359, and the Wireless Networking and Communications Group Industrial Affiliates Program.

\bibliography{bibliography}
\bibliographystyle{icml2021}

\newpage
\onecolumn
\appendix

\section{Computational and Unconditional Hardness}

\hardnessDelays*
\begin{proof}

We consider the simple instance of a single arm, denoted by $i$, which is always feasible to play (if available) and an infinite time horizon. The arm has a deterministic reward equal to $1$, while its delay distribution can be described as:
\begin{align*}
D_{i} = 
      \begin{cases}
        d, & \text{w.p. } p=\frac{1}{2} \\
        1, & \text{w.p. } 1-p,
      \end{cases}
\end{align*}
where $d > 1$ is an integer. Note that in the above instance, we have $d_{\max} = d$.

In the above setting, assuming that the player does not have a priori knowledge of the delay realizations, we consider an online policy, $alg(q)$ defined as follows: Whenever arm $i$ is available, it is played with probability $q$. It is not hard to see that the {\em parameterized} policy $alg(q)$, captures any optimal policy of the player in the above asymptotic scenario. For any $q$, such a policy can be represented using a Markov Chain (MC) (c.f. \cite{P94}) with $d$ states: $0,1, \dots ,d-1$. Each of these states indicates the number of subsequent rounds that remain until $i$ becomes available to play. 
At any round $t$, if the current state is $0$ (i.e., the arm is available), the chain transitions from $0$ to $d-1$ with probability $p\cdot q$. This corresponds to the case where the player chooses to play arm $i$ at time $t$ (an event that happens with probability $q$) and the delay realization at $t$ is $D_{i}=d$ (which happens independently with probability $p$). Alternatively, the chain transitions from state $0$ to itself with probability $(1-q)+q(1-p)=1-p\cdot q$, which corresponds to the case where the player either does not play $i$ at time $t$, or plays $i$ at $t$ but the delay realization is $1$. Finally, the MC transitions from any state $r\geq 1$ to state $r-1$ deterministically, as indicated in \Cref{fig:MC}.

\begin{figure}[ht]
\vskip 0.2in
\begin{center}
\centerline{\includegraphics[width=0.45\columnwidth]{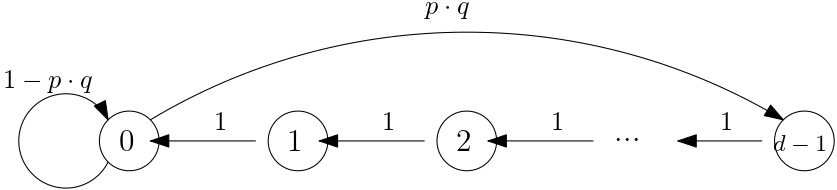}}
\caption{State transitions of policy $alg(q)$.}
\label{fig:MC}
\end{center}
\vskip -0.2in
\end{figure}

We assume that $alg(q)$ starts from state $r\in \{0,1,...,d-1\}$ with probability $\pi(r)$. The stationary distribution $\pi(r)$ corresponds to the solution of the following system equations: $\pi(1)=...=\pi(d-1)=p\cdot q \cdot\pi(0)$ and $\sum_{r=0}^{d-1}\pi(r)=1$. From this system it follows that: $\pi(0)=\frac{1}{1+p\cdot q\cdot (d-1)}$. Due to stationarity, for any time horizon $T>0$, the expected average reward of $alg(q)$ is:
\begin{align*}
    \Rew_{avg}^{alg(q)} = \Ex{}{\frac{1}{T}\sum_{t=1}^{T}\event{\A_t^{alg(q)}=i}} = q\cdot \pi(0) = \frac{q}{1+p\cdot q\cdot (d-1)}.
\end{align*}

Under the assumption that the delay realizations are not known a priori, as $T$ goes to infinity, an optimal online algorithm can be represented using the above MC (c.f. \cite{P94}). To identify this policy we can maximize the expected average reward $\Rew_{avg}^{alg(q)}$ w.r.t. $q \in [0,1]$. The quantity $\Rew_{avg}^{alg(q)}$ is maximized for $q=1$. This verifies the intuition that, if the player is not a priori aware of the realizations of the delay, arm $i$ should be played every time it is available. For $q=1$, the expected average reward collected by $alg(q)$ is: $\Rew_{avg}^{alg(1)}=  \frac{1}{1+p\cdot (d-1)}$. 

We now turn our focus to algorithms that are a priori aware of all the delay realizations. Specifically, we consider the following (possibly sub-optimal) policy: At any time $t$, after observing the realization of $D_{i,t}$, play arm $i$ if and only if $D_{i,t}=1$. The expected average reward of this policy within any time horizon $T$, with $T>0$, is: 
\begin{align*}
    \Rew_{avg}^{cl} = \Ex{}{\frac{1}{T}\sum_{t=1}^{T}\event{\A_t^{cl}=i}} = \Ex{}{\frac{1}{T}\sum_{t=1}^{T}\event{D_{i,t}=1}} = 1-p.
\end{align*}

Thus, for the ratio of the expected average rewards of these two policies, we have that:
\begin{align*}
    \frac{\Rew_{avg}^{alg(1)}}{\Rew_{avg}^{cl}} = \frac{1}{(1-p)(1+p\cdot (d-1))} = \frac{4}{d+1},
\end{align*}
which concludes our proof.
\end{proof}

\hardnessPoly*
\begin{proof}
Suppose there is a polynomial $(\frac{k}{2})^{-\frac{1}{2}+\epsilon}$-approximation algorithm for the \CBBSD~problem, where $k$ is the number of arms. We consider an instance of the {\em Edge-Disjoint Path} (EDP) problem on a {\em directed} graph $\mathcal{G} = (V,E)$ with $|E| = k'$ and a set of $m$ pairs of vertices $\mathcal{T} = \{(s_i, t_i) | s_i,t_i \in V, i \in [m]\}$. By \Cref{thm:edp:gap_problem}, it is NP-hard to distinguish whether all the $m$ pairs or at most $\frac{m}{k'^{1/2 - \epsilon'}}$ pairs of the above instance can be connected by edge-disjoint paths, for any $\epsilon' > 0$. Our goal is to show that our $(\frac{k}{2})^{-\frac{1}{2}+\epsilon}$ approximation algorithm can be used to distinguish between these two cases. 

First, we slightly transform the graph $\mathcal{G}$: We construct a directed graph  $\mathcal{G}'=(V',E')$ such that $V' = V \cup \{s'_1, s'_2, \dots, s'_m\}$ and 
$E' = E \cup \big\{(s'_1, s_1), (s'_2, s_2), \dots, (s'_m, s_m) \big\}$, as indicated in Figure \ref{fig:EDP}.

\begin{figure}[ht]
\vskip 0.2in
\begin{center}
\centerline{\includegraphics[width=0.40\columnwidth]{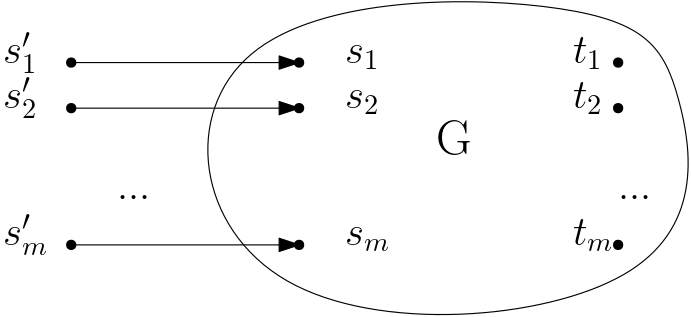}}
\caption{Construction of $G'$ from $G$.}
\label{fig:EDP}
\end{center}
\vskip -0.2in
\end{figure}

We now consider an instance of the \CBBSD~problem where:

\begin{itemize}
    \item Every directed edge of $\mathcal{G}'$ corresponds to an arm, that is, for the set of arms $\A$ we have that $\A = E'$. The number of arms is denoted by $k$, thus $k=|A|=|E'|= k' + m$.
    \item The rewards of all arms are deterministically equal to $0$, except for the arms that correspond to edges $(s'_i, s_i)$ for $i \in [m]$, which have a deterministic reward equal to $1$. Moreover, all arms have deterministic delays equal to $m$.
    \item The feasible sets of arms correspond to edges that form a directed ($s'_i-t_i$) path for a unique $i\in[m]$. That is, 
    $\I = \{S \subseteq \A~|~\exists i \in [m] \text{ unique s.t. }S\text{ forms a directed path from $s'_i$ to $t_i$}\}.$ 
    \item The time horizon is a multiple of $m$, i.e. $T=c\cdot m$ for some (polynomially bounded) $c\in \mathbb{N}_{\geq 1}$.
\end{itemize}

\begin{claim}
We make the following observations on the above instance:
\begin{enumerate}
    \item The family of feasible sets $\I$ does not satisfy the hereditary property. 
    \item Given a subset $\F$ of available arms, a maximum reward available feasible set in $\I(\F)$ can be computed in polynomial-time.
    \item Any feasible arm-pulling schedule can be transformed in polynomial-time into a periodic schedule of period $m$ with at least the same reward. 
\end{enumerate}
\end{claim}
\begin{proof}
The first claim holds trivially, since a subset of a path from $s'_i$ to $t_i$ is not generally an $(s'_i-t_i)$ path.  

For the second claim, we observe that the maximum possible reward collected by any algorithm during any round is $1$. This can be obtained by playing a path from $s'_i$ to $t_i$ for some $i \in [m]$, if such a path exists among the available edges. Given a graph $\mathcal{G}'$ and pairs of vertices $(s'_i,t_i)$ for $i \in [m]$, we can find a path, if one exists, in polynomial time (e.g. simply by performing a DFS for each one of the nodes $s'_i$).  

For the third claim, we consider a solution to the \CBBSD~instance of average reward equal to $\frac{m'}{m}$, for some $m'$. Then, there exists an interval of $m$ consecutive rounds of total reward at least $m'$. Given that $T$ is polynomial in $m$, we can observe all possible sub-intervals of $m$ consecutive rounds and choose the one of maximum total reward. Let $L \subseteq [T]$ with $|L| = m$ be the above sub-interval of total reward at least $m'$. It is not hard to see that by repeating the arm-pulling pattern in $L$ for $c=\frac{T}{m}$ times, consecutively, the resulting solution has average reward at least $\frac{m'}{m}$. It remains to verify that the resulting solution is feasible. First, we observe that each set of arms played at each round is feasible, since $L$ is part of a feasible solution. Moreover, given that the arms of each set selected at every time step are played exactly once in each period of $m$ rounds and the arm delays are all equal to $m$, the resulting schedule does not violate the blocking constraints.
\end{proof}

The above claim allows us to compare the reward of our approximation algorithm to the reward of the optimal algorithm within a period of duration $m$. Focusing on a period of $m$, we distinguish the following two cases:

{\bf Case 1 (Yes instance):} Suppose that in the given EDP instance, all $m$ pairs can be routed by edge-disjoint paths. Then, one possible solution can be obtained by playing the sets of arms corresponding to each of the $m$ paths in a round-robin manner. It is easy to verify that the resulting periodic schedule is feasible and has an average reward of $1$. Let us denote by $\Rew^{*}$ the average reward of an optimal algorithm, and by $\Rew^{apx}$ the corresponding average reward collected by the $(\frac{k}{2})^{-\frac{1}{2}+\epsilon}$-approximation algorithm. By the above discussion, the average reward collected by an optimal algorithm must satisfy $\Rew^{*} \geq 1$. Thus, the approximation algorithm collects an average reward such that:
\begin{align*}
    \Rew^{apx} \geq
    \frac{1}{(\frac{k}{2})^{\frac{1}{2}-\epsilon}} \Rew^{*} 
    \geq
    \frac{1}{(\frac{k}{2})^{\frac{1}{2}-\epsilon}}.
\end{align*}

{\bf Case 2 (No instance):} We now focus on the case where at most a $\frac{1}{k'^{1/2 - \epsilon'}}$-fraction of the pairs in $\mathcal{T}$ can be connected by edge-disjoint paths. Let the number of edge-disjoint paths in $\mathcal{T}$ be $m'$, with $m'\leq \frac{m}{k'^{1/2 - \epsilon'}}$. Then, no algorithm can collect an average reward greater that $\frac{m'}{m}$, because: (i) at every time step, any algorithm can collect a reward of at most $1$, by playing a feasible path, and (ii) the arms that correspond to a feasible path become blocked for the next $m-1$ rounds after they are played. Thus, collecting an average reward greater than $\frac{m'}{m}$ (that is, collecting a reward greater than $m'$ within some period of duration $m$) suggests that there exist more than $m'$ edge-disjoint paths in $\mathcal{T}$. Based on the above observation, the reward collected by the approximation algorithm is:
\begin{align*}
    \Rew^{apx} \leq \Rew^{*} \leq \frac{1}{(k')^{\frac{1}{2}-\epsilon'}} \leq \frac{1}{(\frac{k}{2})^{\frac{1}{2}-\epsilon'}},
\end{align*}
where, for the last inequality, we use that $k' \geq \frac{k}{2}$, since $k = k' + m$ and $m \leq k'$. Finally, by choosing any $\epsilon'>0$ such that $\epsilon'< \epsilon$, we are able to distinguish between the two cases, which concludes our proof. 
\end{proof}

\hardnessDown*
\begin{proof}
In the case where the feasible set of arms satisfies the hereditary property, we establish the hardness of the \CBBSD~problem through a reduction from Max-k-Cover. Consider any instance of Max-k-Cover, where $U = \{e_1, \dots, e_k\}$ is the ground set, $S_1, \dots, S_m \subseteq U$ is the given family of subsets and $l$ is the number of subsets we can collect in order to cover the maximum possible number of elements in $U$.

We consider the following instance of the \CBBSD~problem: 
\begin{itemize}
    \item The set of arms $\A$, with $k=|A|$, contains one arm for each element in $U$. 
    \item Every arm in $\A$ has a deterministic reward of $1$ and a deterministic delay equal to $l$, i.e. the number of subsets we can collect in the Max-k-Cover instance.
    \item The feasible set is defined as
    $
    \I = \{S \subseteq \A~|~\exists i \in [m]~\text{s.t.}~S\subseteq S_i\}
    $, that is, a subset of arms is feasible if it is contained in at least one set of the given family $S_1, \dots, S_m$. 
    \item We define the time horizon to be $T = l \cdot \lceil g(k,m) \rceil$ with $g(k,m) \in poly(k,m)$, that is, an integer multiple of $l$ that is polynomial in $k$ and $m$.
\end{itemize}

Given the above construction, we make the following observations:

\begin{claim} \label{claim:hardness:down}
For the above instance of \CBBSD~problem, we have:
\begin{enumerate}
    \item The family of feasible sets $\I$ satisfies the hereditary property.
    \item Given a subset $\F$ of available arms, computing the maximum reward feasible set in $\I$, that is contained in $\F$, can be achieved in polynomial-time.
    \item Any solution to the above instance can be transformed in polynomial-time into a periodic schedule of period $l$ and with at least the same total reward.
\end{enumerate}
\end{claim}
\begin{proof}
Showing the first claim is straightforward. By definition, for any set $S \in \I$ we have that $S \subseteq S_i$ for some $i \in [m]$. Therefore, any subset $S' \subset S$ is also contained in $S_i$ and, thus, $S' \in \I$. 

For the second claim, given a set of available arms $\F$,  the maximizer of $\max_{i \in [m]} \{|F \cap S_i|\}$ can be computed in polynomial time by trying all sets $S_1, \dots, S_m$ and choose the one of maximum intersection with $F$.

Let us now focus on the third claim. Consider a solution to the \CBBSD~instance of average reward equal to $R$. Clearly, there exists an interval of $l$ consecutive rounds of total reward at least $l \cdot R$. Otherwise, the total reward over $T = l \cdot \lceil g(k,m) \rceil$ rounds should be strictly less than $l \cdot \lceil g(k,m) \rceil \cdot R = T \cdot R$ and, thus, the average reward cannot be equal to $R$. Since $T$ is polynomial in $l, m$ and $k$, we can observe all possible sub-intervals of $l$ consecutive rounds and choose the one of maximum total reward. 

Let $L \subseteq [T]$ with $|L| = l$ be the above sub-interval of total reward at least $l \cdot R$ (and average reward $R$). It is not hard to see that by repeating the arm-pulling pattern in $L$ for $\frac{T}{l}$ times, consecutively, the resulting solution has average reward at least $R$. Thus, it remains to verify that the resulting solution is feasible. Clearly, given that $L$ is part of a feasible solution, the set of arms played at each round is in $\I$. Finally, given that each arm is played exactly once in each period of $l$ rounds and the arm delays are all deterministic and equal to $l$, the resulting solution satisfies the blocking constraints.
\end{proof}

We are now ready to prove our reduction. Let $\OPT$ be the optimal solution of Max-k-Cover and $\Rew^*$ be the solution in the above instance of \CBBSD. 

Clearly, when $\OPT \geq f$ for some integer $f$, then it has to be that $\Rew^* \geq \frac{T}{l} f$. Indeed, let $\{T^*_1, \dots, T^*_l\} \subseteq \{S_1, \dots, S_m\}$ be the optimal solution of Max-k-Cover, such that $|\bigcup_{i \in [l]} T^*_l| \geq f$. Then, we can construct a feasible $l$-periodic solution to the corresponding instance of CBB as follows: In a single $l$-period, the algorithm consecutively plays the arms in $(T^*_1)$, $(T^*_2 \setminus T^*_1)$, $(T^*_3 \setminus T^*_2, T^*_1)$, $\dots$, $(T^*_l\setminus T^*_1, \dots T^*_{l-1})$. Clearly, the resulting solution is feasible and has total reward at least $\frac{T}{l} f$. 

We now consider the case where $\Rew^* \geq \frac{T}{l} f$ for some integer $f$. Then, for the optimal solution of Max-k-Cover, it has to be that $\OPT \geq f$. This holds since, w.l.o.g., we can assume that the optimal solution of \CBBSD~has an $l$-periodic structure, as described in Claim \ref{claim:hardness:down}. Therefore, since the total reward is $\frac{T}{l} f$ and there are exactly $\frac{T}{l}$ periods, the reward of each period must be at least $f$. By focusing on the first period, let $\{\A_1, \dots, \A_l\}$ be arms played at each round from $1$ to $l$. Since $A_t \in \I$ for each $t \in [l]$, let $T_t$ be a set in $\{S_1, \dots, S_m\}$ such that $\A_t \subseteq T_t$. Clearly, the subfamily $\{T_1, \dots, T_l\}$ consists a feasible solution to the Max-k-Cover problem such that $|\bigcup_{t \in [l]} T_l| \geq f$.

By the above analysis we can conclude that, for $\epsilon > 0$, an efficient $\left(1- \frac{1}{e} + \epsilon \right)$-approximation algorithm for \CBBSD~would imply a $\left(1- \frac{1}{e} + \epsilon \right)$-approximation algorithm for Max-k-Cover. However, by Theorem \ref{thm:feige}, we know that this is impossible, unless P = NP. 
\end{proof}
\section{Full-Information Setting}

\stochastic*
\begin{proof}
Recall that we denote by $\A^*_t$ and $\B^*_t$ the set of played and blocked arms, respectively, at any time $t$.
The event that arm $i$ is played at round $t$ and the event that $i$ is blocked at $t$ are mutually exclusive, namely, $
\event{i \in \A_t^*}+ \event{i \in \B_t^{*}}\leq 1$. Further, an arm $i$ is blocked at round $t$ if and only if the arm has been played in some previous time step and is still unavailable at $t$. Therefore, at any round $t$, we get: 
\begin{equation*}
    \event{i \in \A_t^*}+ \sum_{t'<t}\event{i \in \A_{t'}^*}\event{D_{i,t'}>t-t'}\leq 1.
\end{equation*}

Note that the LHS of the inequality above cannot be greater than $1$ because the events are mutually exclusive. Indeed, if an arm has been played at some round $t' < t$ and remains blocked until after round $t$, then this arm cannot have been played at any subsequent time point between $t'$ and $t$ (including $t$). 

By taking expectation in the above expression, we have that: 
\begin{align}
\Pro{i \in \A_t^*}+ \sum_{t'<t}\Ex{}{\event{i \in \A_{t'}^*} \event{D_{i,t'}>t-t'}}\leq 1, \quad \forall t \in [T] \label{inq:lem1:alpha}
\end{align}

For any two rounds $t$ and $t'$ such that $t' < t$, since we assume that the optimal algorithm is not aware of the delay realization of an arm before playing it, the events $\{i \in \A_{t'}^*\}$ and $\{D_{i,t'}>t-t'\}$ are conditionally independent given the history of arms played up to (and including) time $t'-1$. Let us denote the history up to $t$ as $H^*_{t}$, then: 
\begin{align*}
    \Ex{}{\event{i \in \A_{t'}^*} \event{D_{i,t'}>t-t'}} &= \Ex{}{ \Ex{}{\event{i \in \A_{t'}^*} \event{D_{i,t'}>t-t'} ~|~H^*_{t'-1}} }\\
    &= \Ex{}{ \Ex{}{\event{i \in \A_{t'}^*}~|~H^*_{t'-1} } \Ex{}{\event{D_{i,t'}>t-t'} ~|~H^*_{t'-1}} }\\
    &= \Pro{i \in \A_{t'}^*}\Pro{D_{i,t'}>t-t'} .
\end{align*}

By summing \eqref{inq:lem1:alpha} over all rounds $t\in [T]$ and since $\Pro{D_{i,t}>0} = 1~\forall i \in \A \forall t \in [T]$, we get:
\begin{equation}
    \sum_{t\in [T]}\sum_{t'\leq t}\Pro{i \in \A_{t'}^*}
    \Pro{D_{i,t'}>t-t'} = \sum_{t\in [T]}\Pro{i \in \A_t^*}
    + \sum_{t\in [T]}\sum_{t'<t}\Pro{i \in \A_{t'}^*}\Pro{D_{i,t'}>t-t'}\leq T \label{eq:sum_blocked} .
\end{equation}

Since the variables $D_{i,t}$ for $t\in [T]$ are i.i.d., we can consider any sample $D_i$ from the delay distribution of arm $i$ and have $\Pro{D_{i}>t-t'} = \Pro{D_{i,t}>t-t'}$. Moreover, since $D_{i}$ has support in $\{1,...,d_i^{\max}\}$, we can write:

\begin{align*}
    &\sum_{t\in [T]}\sum_{t'\leq t}\Pro{i \in A_{t'}^*}\Pro{D_{i}>t-t'} \\
    &=\sum_{t\in [T]}\sum_{t'=\max\{1,t-d_i^{\max}\}}^{t} 
    \Pro{i \in \A_{t'}^*}\Pro{D_{i}>t-t'}\\
    &= \sum_{t'\in [T]}\sum_{t=t'}^{\min\{T,t'+d_i^{\max}-1\}}
    \Pro{i \in \A_{t'}^*}\Pro{D_{i}>t-t'} \\
    &=\sum_{t'\in [T]} 
    \Pro{i \in \A_{t'}^*}
    \sum_{t=t'}^{\min\{T,t'+d_i^{\max}-1\}}
    \Pro{D_{i}>t-t'} \\
    &\geq \sum_{t'\in [T]} 
    \Pro{i \in \A_{t'}^*}
    \sum_{t=t'}^{t'+d_i^{\max}-1}
    \Pro{D_{i}>t-t'}  -
    \sum_{t'=T-d_i^{\max}+1}^{T}
    \sum_{t=t'}^{T}
    \Pro{D_{i}>t-t'}  \\
    &=\sum_{t'\in [T]} 
    \Pro{i \in \A_{t'}^*}
    \Ex{}{D_i}  -
    \sum_{t'=T-d_i^{\max}+1}^{T}
    \sum_{t=t'}^{T}
    \Pro{D_{i}>t-t'},
\end{align*}
where the inequality is due to upper bounding the probability $\Pro{i \in A_{t'}^*}$ by $1$. Further, in the last equality we use the fact that $D_i$ is a strictly positive random variable with bounded support in $\{1,\dots, d^{\max}_i\}$, thus: $$\Ex{}{D_i} = \sum_{t=t'}^{t'+d_i^{\max}-1}\Pro{D_{i}>t-t'}.$$

Therefore, inequality \eqref{eq:sum_blocked} becomes:
\begin{align*}
    \sum_{t'\in [T]} 
    \Pro{i \in \A_{t'}^*}
    \Ex{}{D_i} -
    \sum_{t'=T-d_i^{\max}+1}^{T}
    \sum_{t=t'}^{T}
    \Pro{D_{i}>t-t'} \leq T .
\end{align*}
Finally, by rearranging terms in the above expression and dividing by $\Ex{}{D_i} \cdot T$, we conclude that:
\begin{align*}
    \frac{1}{T} \sum_{t'\in [T]} \Pro{i \in \A_{t'}^*} 
    &\leq 
    \frac{1}{\Ex{}{D_i}}
    +
    \frac{1}{T \Ex{}{D_i}}
    \sum_{t'=T-d_i^{\max}+1}^{T}
    \sum_{t=t'}^{T}
    \Pro{D_{i}>t-t'}  \\
    &\leq 
    \frac{1}{\Ex{}{D_i}}
    +
    \frac{d_i^{\max} \Ex{}{D_i}}{T \Ex{}{D_i}} \\
    &= 
    \frac{1}{\Ex{}{D_i}} + \frac{d_i^{\max}}{T} ,
\end{align*}
where in the last inequality we use the fact that the second sum $\sum_{t=t'}^{T}\Pro{D_{i}>t-t'}$ can be upper bounded by ${\Ex{}{D_i}}$, for all $t'$ specified in the first sum. 

\end{proof}

\section{Bandit Setting: Omitted Proofs}

\gammaDecomp*
\begin{proof}
Let $H^{\pit}_t$ be the history of arm playing of our bandit algorithm up to (including) time $t$. The expected reward collected by our algorithm can be expressed as:
\begin{align*}
\Ex{}{ \sum_{t \in [T]} \sum_{i \in \A} X_{i,t} \event{i \in \A^{\pit}_t} } 
&= \sum_{t \in [T]}\sum_{i \in \A} \Ex{}{X_{i,t} \event{i \in \A^{\pit}_t} } \\
&= \sum_{t \in [T]}\sum_{i \in \A} \Ex{}{\Ex{}{X_{i,t} \event{i \in \A^{\pit}_t}|H^{\pit}_{t-1}}} \\
&= \sum_{t \in [T]}\sum_{i \in \A} \Ex{}{\Ex{}{X_{i,t}|H^{\pit}_{t-1}} \Ex{}{\event{i \in \A^{\pit}_t}|H^{\pit}_{t-1}}} \\
&= \sum_{t \in [T]}\sum_{i \in \A} \Ex{}{\mu_i \event{i \in \A^{\pit}_t}} \\
&= \sum_{t \in [T]}\Ex{}{\mu(\A^{\pit}_t)},
\end{align*}
where we use the fact that, at each round $t \in [T]$, the reward $X_{i,t}$ and the choice of the set $\A^{\pit}_t$ are independent conditioned on the history $H^{\pit}_t$.

We recall the expected pulling rate of each arm $i \in \A$ by an optimal full information algorithm that satisfies Conditions \ref{assumption:delays} and \ref{assumption:independence}, as defined in \Cref{sec:fullinformation}: $$z_i = \Ex{}{\frac{1}{T}\sum_{t\in [T]}\event{i \in \A_t^*}},$$
where $\A^*_t$ is the set of arms played by the optimal algorithm at time $t$.

A direct application of the definition of $\Gamma^{\pit}_t$, for every $t \in [T]$, gives:
\begin{align*}
    \mu(\A^{\pit}_t) = \alpha \cdot \beta \cdot \mu(\OPT_{\mu}(\F^{\pit}_t)) - \Gamma^{\pit}_t.
\end{align*}

By working along the lines of \Cref{thm:full_information} and carrying the extra term of $- \Gamma^{\pit}_t$, we can see that for any round $t \in [T]$:
\begin{align*}
\Ex{}{\mu(\A^{\pit}_t)}
&= \alpha \cdot \beta \cdot\Ex{}{\mu(\OPT_{\mu}(\F^{\pit}_t))} - \Ex{}{\Gamma^{\pit}_t}\\
&\geq 
\alpha\cdot\beta \cdot \Ex{}{\sum_{i \in \F^{\pit}_t} \mu_i z_i} - \Ex{}{\Gamma^{\pit}_t} \\
&= 
\alpha\cdot\beta \sum_{i \in \A} \mu_i z_i - \alpha\cdot\beta \Ex{}{\sum_{i \in \B_{t}^{\pit}} \mu_i z_i} - \Ex{}{\Gamma^{\pit}_t},
\end{align*}
where the inequality above is due to \Cref{lemma:hereditary} and expectations are taken over all sources of randomness (that is, reward and delay realizations). 

By summing over $t\in T$ and using the fact that $\Ex{}{\sum_{t \in [T]}\sum_{i \in \B^{\pit}_{t}} \mu_i z_i} \leq \Rew^{\pit}(T) + \mathcal{O}(d_{\max}\cdot k)$ (which follows exactly as in \Cref{inq:online:second} in \Cref{thm:full_information}), we have:
\begin{align*}
    \Rew^{\pit}(T)=\Ex{}{\sum_{t\in T}\mu(\A^{\pit}_t)} \geq \alpha\beta \Rew^{*}(T) - \alpha\beta \Rew^{\pit}(T) - \mathcal{O}\left({d_{\max}\cdot k}\right) - \Ex{}{\sum_{t\in T}\Gamma^{\pit}_t} .
\end{align*}
By rearranging terms in the expression above, we get that:
\begin{align*}
    \Rew^{\pit}(T) \geq \frac{\alpha\cdot \beta}{1 + \alpha\cdot \beta} \Rew^{*}_I(T) - \mathcal{O}\left({d_{\max}\cdot k}\right) - \frac{1}{1 + \alpha\cdot \beta} \Ex{}{\sum_{t \in [T]} \Gamma^{\pit}_t}.
\end{align*}
Thus, by definition of the $\rho$-regret, we can conclude that:
\begin{align*}
    \rho\Reg^{\pit}(T) = \frac{\alpha\beta}{1 + \alpha\beta}\Rew^{*}_I(T) - \Rew^{\pit}(T) \leq \frac{1}{1 + \alpha\cdot \beta} \Ex{}{\sum_{t \in [T]} \Gamma^{\pit}_t} + \mathcal{O}(d_{\max}\cdot k).
\end{align*}
\end{proof}

\niceRun*
\begin{proof}
This result is due to \cite{CWYW16}, but we include its proof here for completeness.

For any $T_{i,t-1}>0$ and $i\in \A$, using Hoeffding's inequality (see \cite{MU2005}) and the fact that $T_{i,t-1}\leq t-1$, we have that:
\begin{align*}
    \Pro{|\hat\mu_{i,t-1} - \mu_{i}|\geq \sqrt{\frac{3 \ln{(t)}}{2T_{i,t-1}}}} \leq \sum_{s=1}^{t-1} \Pro{|\hat\mu_{i,t-1} - \mu_{i}|\geq \sqrt{\frac{3 \ln{(t)}}{2s}}} \leq 2te^{-3\ln{t}} \leq \frac{2}{t^2}.
\end{align*}
The lemma follows by taking union bound over $i\in \A$ and using the above inequality:
\begin{align*}
    \Pro{\neg \mathcal{N}_t} &= \Pro{\exists i\in \A, ~s.t.~|\hat\mu_{i,t-1} - \mu_{i}| > \sqrt{\frac{3 \ln{(t)}}{2T_{i,t-1}}}} \\ 
    &\leq \sum_{i \in \A} \Pro{|\hat\mu_{i,t-1} - \mu_{i}| > \sqrt{\frac{3 \ln{(t)}}{2T_{i,t-1}}}} \\
    &\leq \frac{2k}{t^2} ~.
\end{align*}

\end{proof}

\sumOfKappa*
\begin{proof}
The proof of this lemma follows closely the analysis of \cite{WC17} for the case of a fixed availability set. However, in our case the availability set at every round is correlated with previous actions. To tackle this issue, we make the following relaxation:  For each arm $i$, we consider the minimum possible suboptimality gap associated with arm $i$, $\Delta_i^{\min}$ (see \Cref{def:differences}). This allows us to keep track of the number of rounds where our bandit algorithm can select a bad feasible solution containing arm $i$ (w.r.t. the availability set of the round), by examining whether $T_{i}\leq  \ell_T(\Delta_{\min}^i)$. This relaxation allows us to eventually drop the dependence on the availability set. 

Observe that, by definition of $\mathcal{N}_t$, at any time $t$, if $\mathcal{N}_t$ holds, then $\forall{i \in \A}$ the UCB indices at round $t$ can be bounded as follows:
\begin{align}
    |\mub_{i,t-1}-\mu_i|\leq \min\Big\{2\sqrt{\frac{3\ln{t}}{2T_{i,t-1}}},1\Big\},
    \label{eq:nice_sampling}
\end{align} 
and: 
\begin{align} \label{eq:overestimators}
    \mub_{i,t-1} \geq \mu_i .
\end{align}
Then, if $\mathcal{Q}_t$ holds, we have that:
\begin{align*}
    \mub(\A^{\pit}_t) 
    \geq \alpha\cdot\mub(\OPT_{\mub}(\F^{\pit}_t)) 
    \geq \alpha\cdot\mub(\OPT_{\mu}(\F^{\pit}_t)) ,
\end{align*}
where for the last inequality we use the definition of $\OPT_{\mub}$.
By applying \Cref{eq:overestimators}, we get that:
\begin{align*}
    \mub(\A^{\pit}_t) 
    \geq \alpha\cdot\mub(\OPT_{\mu}(\F^{\pit}_t)) \geq
    \alpha\cdot\mu(\OPT_{\mu}(\F^{\pit}_t)) = \mu(\A^{\pit}_t) + \Delta(\A^{\pit}_t,\F^{\pit}_t),
\end{align*}
where for the last equality we use the definition of $\Delta(\A^{\pit}_t,\F^{\pit}_t)$.

By reordering terms, we get that:
\begin{align}\label{eq:sum_of_mus}
    \Delta(\A^{\pit}_t,\F^{\pit}_t) \leq \mub(\A^{\pit}_t) - \mu(\A^{\pit}_t) \leq \sum_{i\in \A^{\pit}_t} |\mub_i-\mu_i| .
\end{align}
The event $\A^{\pit}_t\in \mathcal{S}_B(\F^{\pit}_t)$ dictates that a bad feasible solution has been played at round $t$. Then, using the definition of $\Delta_{\min}^i$, for the suboptimality gap of the set selected at round $t$ we have that:  $$\Delta(\A^{\pit}_t,\F^{\pit}_t)\geq \max_{i\in\A^{\pit}_t}\Delta_{\min}^i.$$ 
If we apply this inequality to \Cref{eq:sum_of_mus}, we get: 
\begin{align}\label{eq:positive_difference}
    \sum_{i\in \A^{\pit}_t} |\mub_i-\mu_i| - \max_{i\in\A^{\pit}_t}\Delta_{\min}^i \geq 0 .
\end{align} 
Now, using \Cref{eq:sum_of_mus,eq:positive_difference} we can bound the suboptimality gap of the bad set selected at round $t$ as follows:
\begin{align*}
    \Delta(\A^{\pit}_t,\F^{\pit}_t) &\leq \sum_{i\in \A^{\pit}_t} |\mub_i-\mu_i| \leq 2 \sum_{i\in \A^{\pit}_t} |\mub_i-\mu_i| - \max_{i\in\A^{\pit}_t}\Delta_{\min}^i \leq 2 \sum_{i\in \A^{\pit}_t} (|\mub_i-\mu_i| - \frac{\Delta_{\min}^i}{2|\A^{\pit}_t|}) \leq 2 \sum_{i\in \A^{\pit}_t} \Big(|\mub_i-\mu_i| - \frac{\Delta_{\min}^i}{2r}\Big) ,
\end{align*}
where in the last inequality we use that $|\A^{\pit}_t|\leq r$, by definition of $r = \max_{S \in \I}\{|S|\}$. Using \Cref{eq:nice_sampling}, the expression above becomes:
\begin{align*}
    \Delta(\A^{\pit}_t,\F^{\pit}_t) 
    &\leq 2 \sum_{i\in \A^{\pit}_t} \Big(\min\{2\sqrt{\frac{3\ln{t}}{2T_{i,t-1}}},1\} - \frac{\Delta_{\min}^i}{2r}\Big) 
    = \sum_{i\in \A^{\pit}_t} \Big(\min\{\sqrt{\frac{24\ln{t}}{T_{i,t-1}}},2\} - \frac{\Delta_{\min}^i}{r}\Big)  ~.
\end{align*}
Now, observe that if $T_{i,t-1} \leq \ell_{T}(\Delta_{\min}^i)$, then $\min\{\sqrt{\frac{24\ln{t}}{T_{i,t-1}}},2\} - \frac{\Delta_{\min}^i}{r} \leq \min\{\sqrt{\frac{24\ln{t}}{T_{i,t-1}}},2\} = \kappa_T(\Delta_{\min}^i,T_{i,t-1}) $.
On the other hand, if $T_{i,t-1} > \ell_{T}(\Delta_{\min}^i)$, then $\min\{\sqrt{\frac{24\ln{t}}{T_{i,t-1}}},2\} \leq \sqrt{\frac{24\ln{t}}{T_{i,t-1}}} \leq \sqrt{\frac{24\ln{t}}{24  \ln{T}r^2/(\Delta_{\min}^i)^2}}\leq \frac{\Delta_{\min}^i}{r}$, thus $\min\{\sqrt{\frac{24\ln{t}}{T_{i,t-1}}},2\} - \frac{\Delta_{\min}^i}{r}\leq 0 = \kappa_T(\Delta_{\min}^i,T_{i,t-1})$. Therefore, we conclude that:
\begin{align*}
    \Delta(\A^{\pit}_t,\F^{\pit}_t) \leq  \sum_{i\in \A^{\pit}_t} \kappa_T(\Delta_{\min}^i,T_{i,t-1}) .
\end{align*}

\end{proof}

\thmRegret*

\begin{proof}
We first focus on the distribution-dependent bound of the regret. According to \Cref{lemma:gamma}, to bound the $\rho$-approximate regret of our bandit algorithm, it suffices to bound the loss accumulated due to the instantaneous regret over time. This loss can be rewritten as:
\begin{align}
    \Ex{}{\sum_{t \in [T]} \Gamma^{\pit}_t} &= \Ex{}{\sum_{t \in [T]} \Big( \alpha \cdot \beta \cdot \mu(\OPT_{\mu}(\F^{\pit}_t)) - \mu(\A^{\pit}_t) \Big)} \nonumber\\
    &= \Ex{}{\sum_{t \in [T]} \Big(\alpha \cdot \mu(\OPT_{\mu}(\F^{\pit}_t)) - \mu(\A^{\pit}_t) + \alpha \cdot (\beta-1) \cdot \mu(\OPT_{\mu}(\F^{\pit}_t)) \Big)} \nonumber.
\end{align}
Focusing on the term $\alpha \cdot \mu(\OPT_{\mu}(\F^{\pit}_t)) - \mu(\A^{\pit}_t)$ of the loss, observe that this quantity is positive only when the algorithm plays a bad feasible solution w.r.t. the availability set $\F^{\pit}_t$. Thus, we can write:
\begin{align}
    \Ex{}{\sum_{t \in [T]} \Gamma^{\pit}_t} 
    &= \Ex{}{\sum_{t \in [T]} \Big( \alpha \cdot \mu(\OPT_{\mu}(\F^{\pit}_t)) - \mu(\A^{\pit}_t) + \alpha \cdot (\beta-1) \cdot \mu(\OPT_{\mu}(\F^{\pit}_t)) \Big)} \nonumber\\
    &\leq \Ex{}{\sum_{t \in [T]} \Big(\event{\A^{\pit}_t \in \mathcal{S}_B(\F^{\pit}_t)}(\alpha \cdot \mu(\OPT_{\mu}(\F^{\pit}_t) - \mu(\A^{\pit}_t))) + \alpha \cdot (\beta-1) \cdot \mu(\OPT_{\mu}(\F^{\pit}_t)) \Big)} \nonumber\\
    &= \Ex{}{\sum_{t \in [T]} \Big(\event{\A^{\pit}_t \in \mathcal{S}_B(\F^{\pit}_t)}\Delta(\A^{\pit}_t,\F^{\pit}_t) + \alpha \cdot (\beta-1) \cdot \mu(\OPT_{\mu}(\F^{\pit}_t)) \Big)} \label{eq:bad_set_case}.
\end{align}

Now, we can use the idea of \cite{WC17} to treat the parts of the regret that result from oracle fails or non-representative sampling separately. However, in our case, we also have to address the fact that our availability set changes over time, depending on past actions. We decompose the indicator in \Cref{eq:bad_set_case} based on whether the events $\mathcal{Q}_t$ and $\mathcal{N}_t$ hold, as follows: 
\begin{align}
    \eqref{eq:bad_set_case}
    &\leq \Ex{}{\sum_{t \in [T]} \event{\mathcal{Q}_t,\mathcal{N}_t,\A^{\pit}_t\in \mathcal{S}_B(\F^{\pit}_t)}\cdot\Delta(\A^{\pit}_t,\F^{\pit}_t) } \label{eq:QNBad}\\
    &+ \Ex{}{\sum_{t \in [T]} \event{\neg \mathcal{Q}_t,\A^{\pit}_t\in \mathcal{S}_B(\F^{\pit}_t)}\cdot\Delta(\A^{\pit}_t,\F^{\pit}_t) + \alpha \cdot (\beta-1) \cdot \mu(\OPT_{\mu}(\F^{\pit}_t))} \label{eq:oracle_fail}\\
    &+\Ex{}{\sum_{t \in [T]} \event{\neg \mathcal{N}_t,\A^{\pit}_t\in \mathcal{S}_B(\F^{\pit}_t)}\cdot\Delta(\A^{\pit}_t,\F^{\pit}_t)} \label{eq:samples_fail}.
\end{align}

We bound each one of the above terms, separately. For \eqref{eq:oracle_fail}, to address the issue that the availability set changes at each round, we consider the suboptimality gap of the worst feasible solution contained in the availability set $\Sc$,  $\Delta_{max}(\Sc)$, namely:
\begin{equation*}
    \Delta_{\max}(\Sc) = \max_{i\in \A,~S\in S_{i,B}(\Sc)}\Delta(S,\Sc) .
\end{equation*}
We use the above definition to upper bound the suboptimality gap, $\Delta(\A^{\pit}_t,\F^{\pit}_t)$, appearing in \Cref{eq:oracle_fail}. We have that:
\begin{align}
    \eqref{eq:oracle_fail} &\leq \Ex{}{\sum_{t \in [T]} \Big( \event{\neg \mathcal{Q}_t}\cdot \Delta_{\max}(\F^{\pit}_t) + \alpha \cdot (\beta-1) \cdot \mu(\OPT_{\mu}(\F^{\pit}_t))\Big)} \nonumber\\
    &= \Ex{}{\sum_{t \in [T]}\sum_{\Sc\subseteq \A} \Ex{}{\event{\neg \mathcal{Q}_t}\cdot \Delta_{\max}(\Sc) + \alpha \cdot (\beta-1) \cdot \mu(\OPT_{\mu}(\Sc)) ~|~ \F^{\pit}_t = \Sc } \event{\F^{\pit}_t = \Sc}} \nonumber\\
    &= \Ex{}{\sum_{t \in [T]}\sum_{\Sc\subseteq \A} \Big(\Pro{\neg \mathcal{Q}_t ~|~ \F^{\pit}_t = \Sc}\cdot \Delta_{\max}(\Sc) + \alpha \cdot (\beta-1) \cdot \mu(\OPT_{\mu}(\Sc))\Big)\event{\F^{\pit}_t = \Sc} } \nonumber\\
    &\leq \Ex{}{\sum_{t \in [T]}\sum_{\Sc\subseteq \A} \Big((1-\beta)\cdot \Delta_{\max}(\Sc) + \alpha \cdot (\beta-1) \cdot \mu(\OPT_{\mu}(\Sc))\Big)\event{\F^{\pit}_t = \Sc} } \nonumber\\
    &\leq 0  \label{eq:oracle_fail_res},
\end{align}
where to obtain the first equality we use the law of total expectation. The second inequality holds since the success probability of the oracle is at least $\beta$ independently of the availability set, thus, $\Pro{\neg \mathcal{Q}_t ~|~ \F^{\pit}_t = \Sc} \leq 1-\beta$. Finally, the last inequality follows by the fact that, for the suboptimality gap of any bad feasible set $S$ (w.r.t. an availability set $\Sc$), we have $\Delta(S,\Sc)\leq \alpha\cdot \mu(\OPT_{\mu}(\Sc))$, thus, $\Delta_{\max}(\Sc)\leq \alpha\cdot \mu(\OPT_{\mu}(\Sc))$.

For the term \eqref{eq:samples_fail}, using the definition of $\Delta_{\max}$ and \Cref{lemma:nice_run}, we have that:
\begin{align}
    \eqref{eq:samples_fail} \leq \Ex{}{\sum_{t \in [T]} \event{\neg \mathcal{N}_t}\cdot\Delta_{\max} } = \sum_{t \in [T]} \Pro{\neg \mathcal{N}_t}\cdot\Delta_{\max} \leq \sum_{t\in [T]} \frac{2 k}{t^2} \Delta_{\max} \leq \frac{\pi^2 k}{3} \Delta_{\max} ~.
    \label{eq:samples_fail_res}
\end{align}

The term in \eqref{eq:QNBad} can be bounded by utilizing \Cref{lemma:kappas}:
\begin{align}
    \eqref{eq:QNBad} &= \Ex{}{\sum_{t \in [T]} \event{\mathcal{Q}_t,\mathcal{N}_t,\A^{\pit}_t\in \mathcal{S}_B(\F^{\pit}_t)}\cdot\Delta(\A^{\pit}_t,\F^{\pit}_t) }\nonumber\\
    &\leq \Ex{}{\sum_{t \in [T]}\sum_{i\in\A^{\pit}_t} \event{\mathcal{Q}_t,\mathcal{N}_t,\A^{\pit}_t\in \mathcal{S}_B(\F^{\pit}_t)} \kappa_T(\Delta_{\min}^i,T_{i,t-1})} \nonumber\\
    &\leq \Ex{}{\sum_{i\in\A}\sum_{s =0}^{T_{i,T}} \kappa_T(\Delta_{\min}^i,s)} , \label{eq:QNBad2}
\end{align}
where the last inequality holds because a bad set containing arm $i$ can be played at most $T_{i,T}$ times. Then, we substitute $\kappa_T(\Delta_{\min}^i,s)$ and use an integral to upper bound the inner sum, following the idea of \cite{WC17}:
\begin{align}
    \eqref{eq:QNBad2} 
    &\leq \Ex{}{\sum_{i\in\A}\sum_{s =0}^{\ell_{T}(\Delta_{\min}^i)} \kappa_T(\Delta_{\min}^i,s)} = 2k + \sum_{i\in\A}\sum_{s =1}^{\ell_{T}(\Delta_{\min}^i)} \sqrt{\frac{24 \ln(T)}{s}} \leq 2k + \sum_{i\in\A}\int_{s =0}^{\ell_{T}(\Delta_{\min}^i)} \sqrt{\frac{24 \ln(T)}{s}}\,ds \nonumber\\
    &\leq 2k + \sum_{i\in\A} 2\sqrt{24 \ln(T) {\ell_{T}(\Delta_{\min}^i)}} \leq 2k + \sum_{i\in\A} 2\sqrt{24 \ln(T) \frac{24 r^2 \ln{(T)}}{(\Delta_{\min}^i)^2}} \leq 2k + \sum_{i\in\A} 48 \frac{r}{\Delta_{\min}^i}\ln{(T)} ~. \label{eq:QNBad2_res}
\end{align}
We conclude our proof for the distribution-dependent bound by combining the above results with \Cref{lemma:gamma}.

We now focus on the distribution-independent bound of the regret. The following part is similar to the proof for the distribution-dependent bound. We rely on \Cref{lemma:gamma} to bound the $\rho$-approximate regret by bounding the expected loss due to the instantaneous regret over time. By \Cref{eq:bad_set_case} we have that:
\begin{align*}
    \Ex{}{\sum_{t \in [T]} \Gamma^{\pit}_t} 
    \leq \Ex{}{\sum_{t \in [T]} \Big( \event{\A^{\pit}_t \in \mathcal{S}_B(\F^{\pit}_t)}\Delta(\A^{\pit}_t,\F^{\pit}_t) + \alpha \cdot (\beta-1) \cdot \mu(\OPT_{\mu}(\F^{\pit}_t)) \Big) } .
\end{align*}

As before, we distinguish the following cases based on whether the events ${\mathcal{Q}_t}$ and ${\mathcal{N}_t}$ hold:

\begin{align}
    \Ex{}{\sum_{t \in [T]} \Gamma^{\pit}_t} 
    &\leq \Ex{}{\sum_{t \in [T]} \event{\mathcal{Q}_t,\mathcal{N}_t,\A^{\pit}_t\in \mathcal{S}_B(\F^{\pit}_t)}\cdot\Delta(\A^{\pit}_t,\F^{\pit}_t) } \label{eq:QNBad2a}\\
    &+ \Ex{}{\sum_{t \in [T]} \Big( \event{\neg \mathcal{Q}_t,\A^{\pit}_t\in \mathcal{S}_B(\F^{\pit}_t)}\cdot\Delta(\A^{\pit}_t,\F^{\pit}_t) + \alpha \cdot (\beta-1) \cdot \mu(\OPT_{\mu}(\F^{\pit}_t)) \Big) } \label{eq:oracle_fail2}\\
    &+\Ex{}{\sum_{t \in [T]} \event{\neg \mathcal{N}_t,\A^{\pit}_t\in \mathcal{S}_B(\F^{\pit}_t)}\cdot\Delta(\A^{\pit}_t,\F^{\pit}_t)} \label{eq:samples_fail2}.
\end{align}

The terms \eqref{eq:oracle_fail2} and \eqref{eq:samples_fail2} can be bounded using \Cref{eq:oracle_fail_res,eq:samples_fail_res}, respectively. As in \cite{WC17}, we define $M = \sqrt{(48kr\ln{T})/T}$. We use this definition to rewrite the term \eqref{eq:QNBad2a} as:
\begin{align}
    \eqref{eq:QNBad2a} =& \Ex{}{\sum_{t \in [T]} \event{\mathcal{Q}_t,\mathcal{N}_t,\A^{\pit}_t\in \mathcal{S}_B(\F^{\pit}_t), \Delta(\A^{\pit}_t,\F^{\pit}_t) \geq M}\cdot\Delta(\A^{\pit}_t,\F^{\pit}_t) } \label{eq:QNBad2a1}\\
    &+\Ex{}{\sum_{t \in [T]} \event{\mathcal{Q}_t,\mathcal{N}_t,\A^{\pit}_t\in \mathcal{S}_B(\F^{\pit}_t), \Delta(\A^{\pit}_t,\F^{\pit}_t)< M}\cdot\Delta(\A^{\pit}_t,\F^{\pit}_t) }. \label{eq:QNBad2a2}
\end{align}
The second term of the above equation can be bounded using the fact that $\Delta(\A^{\pit}_t,\F^{\pit}_t) < M$, as follows:
\begin{align}
    \eqref{eq:QNBad2a2} \leq \Ex{}{\sum_{t \in [T]} \event{\A^{\pit}_t\in \mathcal{S}_B(\F^{\pit}_t),\Delta(\A^{\pit}_t,\F^{\pit}_t) < M}\cdot\Delta(\A^{\pit}_t,\F^{\pit}_t) } \leq T\cdot M = T \sqrt{(48 \cdot r\cdot k\ln{T})/T} = \sqrt{48\cdot r\cdot k\cdot T\ln{T}}. \label{eq:QNBad2a2_res}
\end{align}
The first term can be bounded using \Cref{lemma:kappas2}, following the same procedure as in the distribution-dependent part:
\begin{align}
    \eqref{eq:QNBad2a1} &= \Ex{}{\sum_{t \in [T]} \event{\mathcal{Q}_t,\mathcal{N}_t,\A^{\pit}_t\in \mathcal{S}_B(\F^{\pit}_t),\Delta(\A^{\pit}_t,\F^{\pit}_t) \geq M}\cdot\Delta(\A^{\pit}_t,\F^{\pit}_t) }\nonumber\\
    &\leq \Ex{}{\sum_{t \in [T]}\sum_{i\in\A^{\pit}_t} \event{\mathcal{Q}_t,\mathcal{N}_t,\A^{\pit}_t\in \mathcal{S}_B(\F^{\pit}_t),\Delta(\A^{\pit}_t,\F^{\pit}_t) \geq M} \kappa_T(M,T_{i,t-1})} \nonumber\\
    &\leq \Ex{}{\sum_{i\in\A}\sum_{s =0}^{T_{i,T}} \kappa_T(M,s)} \nonumber\\
    &\leq 2k + \sum_{i\in\A} 48 \frac{r}{M}\ln{T} \nonumber\\
    &= 2k + \sqrt{48\cdot k\cdot r\cdot T\ln{T}} ,
    \label{eq:QNBad2b}
\end{align}
where the last inequality is derived similarly to \eqref{eq:QNBad2_res}.
We conclude our proof by combining the above \Cref{eq:oracle_fail_res,eq:samples_fail_res,eq:QNBad2a2_res,eq:QNBad2b} with \Cref{lemma:gamma}:
\begin{align*}
    \rho\Reg^{\pit}(T) &\leq \frac{1}{1 + \alpha\cdot \beta} \Ex{}{\sum_{t \in [T]} \Gamma^{\pit}_t} + \mathcal{O}(d_{\max}\cdot k) \\
    &\leq \frac{1}{1 + \alpha\cdot \beta}\Big( \frac{\pi^2 k}{3} \Delta_{\max} + \sqrt{48\cdot k\cdot r\cdot T\ln{T}} + 2k + \sqrt{48\cdot k\cdot r\cdot T\ln{T}} \Big) + \mathcal{O}(d_{\max}\cdot k)\\
    &\leq \frac{14\sqrt{k\cdot r\cdot T\ln{T}}}{1 + \alpha\cdot \beta} + \frac{k}{1 + \alpha\cdot \beta} (2+ \frac{\pi^2 }{3} \Delta_{\max})  + \mathcal{O}(d_{\max}\cdot k) .
\end{align*}
\end{proof}

\begin{lemma} \label{lemma:kappas2}
For any $t\in [T]$, if the event $\{\mathcal{Q}_t, \mathcal{N}_t, \A^{\pit}_t\in \mathcal{S}_B(\F^{\pit}_t), \Delta(\A^{\pit}_t,\F^{\pit}_t) \geq M\}$ holds, then $\Delta(\A^{\pit}_t,\F^{\pit}_t)\leq \sum_{i\in\A^{\pit}_t} \kappa_T(M,T_{i,t-1})$.
\end{lemma}

\begin{proof}
This proof resembles the proof of \Cref{lemma:kappas}: From \Cref{eq:sum_of_mus} and the fact that $\Delta(\A^{\pit}_t,\F^{\pit}_t) \geq M$ we obtain: $\Delta(\A^{\pit}_t,\F^{\pit}_t) \leq 2 \sum_{i\in \A^{\pit}_t} |\mub_i-\mu_i| - M$. This can be rewritten as:
\begin{align*}
    \Delta(\A^{\pit}_t,\F^{\pit}_t) &
    \leq 2 \sum_{i\in \A^{\pit}_t} |\mub_i-\mu_i| - M
    \leq 2 \sum_{i\in \A^{\pit}_t} (|\mub_i-\mu_i| - \frac{M}{2|\A^{\pit}_t|}) \leq 2 \sum_{i\in \A^{\pit}_t} \Big(|\mub_i-\mu_i| - \frac{M}{2r}\Big) ,
\end{align*}
where in the last inequality we use that $|\A^{\pit}_t|\leq r$, by definition of $r$. Then, using \Cref{eq:nice_sampling}, the expression above becomes:
\begin{align*}
    \Delta(\A^{\pit}_t,\F^{\pit}_t) 
    &\leq 2 \sum_{i\in \A^{\pit}_t} \Big(\min\{2\sqrt{\frac{3\ln{t}}{2T_{i,t-1}}},1\} - \frac{M}{2r}\Big) 
    = \sum_{i\in \A^{\pit}_t} \Big(\min\{\sqrt{\frac{24\ln{t}}{T_{i,t-1}}},2\} - \frac{M}{r}\Big) .
\end{align*}
Finally, similarly to \Cref{lemma:kappas}, we note that when $T_{i,t-1} \leq \ell_{T}(M)$, then $\min\{\sqrt{\frac{24\ln{t}}{T_{i,t-1}}},2\} - \frac{M}{r} \leq \kappa_T(M,T_{i,t-1}) $. However, when $T_{i,t-1} > \ell_{T}(M)$ then $\min\{\sqrt{\frac{24\ln{t}}{T_{i,t-1}}},2\} \leq \sqrt{\frac{24\ln{t}}{T_{i,t-1}}} \leq \frac{M}{r}$, thus $\min\{\sqrt{\frac{24\ln{t}}{T_{i,t-1}}},2\} - \frac{M}{r}\leq 0 = \kappa_T(M,T_{i,t-1})$, which completes the proof.

\end{proof}

\end{document}